\RequirePackage[l2tabu,orthodox]{nag}
\documentclass
[letterpaper,11pt,]
{article}

\usepackage{etex}
\usepackage{verbatim}
\usepackage{xspace,enumerate}
\usepackage[dvipsnames]{xcolor}
\usepackage[T1]{fontenc}
\usepackage[full]{textcomp}
\usepackage[american]{babel}
\usepackage{mathtools}
\usepackage{amsthm}
\usepackage[
letterpaper,
top=1in,
bottom=1in,
left=1in,
right=1in]{geometry}
\usepackage{newpxtext} %
\usepackage{textcomp} %
\usepackage[varg,bigdelims]{newpxmath}
\usepackage[scr=rsfso]{mathalfa}%
\usepackage{bm} %
\let\mathbb\varmathbb
\usepackage{microtype}
\usepackage[pagebackref,colorlinks=true,urlcolor=blue,linkcolor=blue,citecolor=OliveGreen]{hyperref}
\usepackage[capitalise,nameinlink]{cleveref}
\crefname{lemma}{Lemma}{Lemmas}
\crefname{fact}{Fact}{Facts}
\crefname{theorem}{Theorem}{Theorems}
\crefname{corollary}{Corollary}{Corollaries}
\crefname{claim}{Claim}{Claims}
\crefname{example}{Example}{Examples}
\crefname{algorithm}{Algorithm}{Algorithms}
\crefname{problem}{Problem}{Problems}
\crefname{definition}{Definition}{Definitions}
\crefname{exercise}{Exercise}{Exercises}
\crefname{condition}{Condition}{Conditions}
\usepackage{amsthm}

\newtheorem{theorem}{Theorem}[section]
\newtheorem*{theorem*}{Theorem}
\newtheorem{lemma}[theorem]{Lemma}
\newtheorem*{lemma*}{Lemma}
\newtheorem{fact}[theorem]{Fact}
\newtheorem*{fact*}{Fact}

\newtheorem*{proposition*}{Proposition}

\newtheorem*{corollary*}{Corollary}

\newtheorem*{hypothesis*}{Hypothesis}

\newtheorem*{conjecture*}{Conjecture}
\theoremstyle{definition}
\newtheorem{definition}[theorem]{Definition}
\newtheorem*{definition*}{Definition}

\newtheorem*{construction*}{Construction}

\newtheorem*{example*}{Example}

\newtheorem*{question*}{Question}
\newtheorem{algorithm}[theorem]{Algorithm}
\newtheorem*{algorithm*}{Algorithm}
\newtheorem{assumption}[theorem]{Assumption}
\newtheorem*{assumption*}{Assumption}
\newtheorem{problem}[theorem]{Problem}
\newtheorem*{problem*}{Problem}

\newtheorem*{openquestion*}{Open Question}
\theoremstyle{remark}

\newtheorem*{claim*}{Claim}

\newtheorem*{remark*}{Remark}

\newtheorem*{observation*}{Observation}
\usepackage{paralist}
\let\originalleft\left
\let\originalright\right
\renewcommand{\left}{\mathopen{}\mathclose\bgroup\originalleft}
\renewcommand{\right}{\aftergroup\egroup\originalright}
\usepackage{turnstile}
\usepackage{mdframed}
\usepackage{tikz}
\usetikzlibrary{positioning}
\usepackage{caption}
\DeclareCaptionType{Algorithm}
\usepackage{newfloat}
\usepackage{array}
\usepackage{subfig}
\usepackage{bbm}
\usepackage{xparse}
\usepackage{amsthm} %
\makeatletter
\let\latexparagraph\paragraph
\RenewDocumentCommand{\paragraph}{som}{%
  \IfBooleanTF{#1}
    {\latexparagraph*{#3}}
    {\IfNoValueTF{#2}
       {\latexparagraph{\maybe@addperiod{#3}}}
       {\latexparagraph[#2]{\maybe@addperiod{#3}}}%
  }%
}
\newcommand{\maybe@addperiod}[1]{%
  #1\@addpunct{.}%
}
\makeatother

\newcommand{\Authornotecolored}[3]{}
\newcommand{\Authorcomment}[2]{}
\newcommand{\Authorfnote}[2]{}

\usepackage{boxedminipage}
\newcommand{\paren}[1]{(#1)}
\newcommand{\Paren}[1]{\left(#1\right)}

\newcommand{\Brac}[1]{\left[#1\right]}

\newcommand{\abs}[1]{\lvert#1\rvert}
\newcommand{\Abs}[1]{\left\lvert#1\right\rvert}

\newcommand{\Set}[1]{\left\{#1\right\}}

\newcommand{\norm}[1]{\lVert#1\rVert}
\newcommand{\Norm}[1]{\left\lVert#1\right\rVert}

\newcommand{\snorm}[1]{\norm{#1}^2}

\newcommand{\iprod}[1]{\langle#1\rangle}

\newcommand{\Psymb}{\mathbb{P}}

\newcommand{\given}{\mathrel{}\middle\vert\mathrel{}}
\newcommand{\suchthat}{\;\middle\vert\;}

\newcommand{\from}{\colon}

\newcommand\bdot\bullet

\DeclareMathOperator{\Ind}{\mathbf 1}

\DeclareMathOperator{\OPT}{OPT}

\DeclareMathOperator{\poly}{poly}

\DeclareMathOperator{\argmin}{argmin}

\DeclareMathOperator{\supp}{supp}

\DeclareMathOperator{\sign}{sign}

\newcommand{\Z}{\mathbb Z}
\newcommand{\N}{\mathbb N}
\newcommand{\R}{\mathbb R}

\newcommand{\cA}{\mathcal A}

\newcommand{\cC}{\mathcal C}

\newcommand{\cS}{\mathcal S}

\newcommand{\cU}{\mathcal U}

\newcommand{\cX}{\mathcal X}

\newcommand{\bbS}{\mathbb S}

\renewcommand{\leq}{\leqslant}

\renewcommand{\geq}{\geqslant}

\let\epsilon=\varepsilon
\numberwithin{equation}{section}
\newcommand\MYcurrentlabel{xxx}
\newcommand{\MYstore}[2]{%
  \global\expandafter \def \csname MYMEMORY #1 \endcsname{#2}%
}
\newcommand{\MYload}[1]{%
  \csname MYMEMORY #1 \endcsname%
}
\newcommand{\MYnewlabel}[1]{%
  \renewcommand\MYcurrentlabel{#1}%
  \MYoldlabel{#1}%
}
\newcommand{\MYdummylabel}[1]{}
\newcommand{\torestate}[1]{%
  \let\MYoldlabel\label%
  \let\label\MYnewlabel%
  #1%
  \MYstore{\MYcurrentlabel}{#1}%
  \let\label\MYoldlabel%
}
\newcommand{\restatetheorem}[1]{%
  \let\MYoldlabel\label
  \let\label\MYdummylabel
  \begin{theorem*}[Restatement of \cref{#1}]
    \MYload{#1}
  \end{theorem*}
  \let\label\MYoldlabel
}
\newcommand{\restatelemma}[1]{%
  \let\MYoldlabel\label
  \let\label\MYdummylabel
  \begin{lemma*}[Restatement of \cref{#1}]
    \MYload{#1}
  \end{lemma*}
  \let\label\MYoldlabel
}
\newcommand{\restateprop}[1]{%
  \let\MYoldlabel\label
  \let\label\MYdummylabel
  \begin{proposition*}[Restatement of \cref{#1}]
    \MYload{#1}
  \end{proposition*}
  \let\label\MYoldlabel
}
\newcommand{\restatefact}[1]{%
  \let\MYoldlabel\label
  \let\label\MYdummylabel
  \begin{fact*}[Restatement of \cref{#1}]
    \MYload{#1}
  \end{fact*}
  \let\label\MYoldlabel
}
\newcommand{\restate}[1]{%
  \let\MYoldlabel\label
  \let\label\MYdummylabel
  \MYload{#1}
  \let\label\MYoldlabel
}

\newcommand{\sse}{\subseteq}

\newcommand{\e}{\epsilon}

\allowdisplaybreaks
\DeclareMathOperator{\Span}{Span}

\newcommand{\err}{\mathrm{err}}

\newcommand{\clwed}[3]{\mathrm{C}_{#1,#2,#3}}
\newcommand{\hclwed}[4]{\mathrm{H}_{#1,#2,#3,#4}}
\newcommand{\nhclwed}[4]{\mathrm{NH}_{#1,#2,#3,#4}}

\newcommand{\clwem}[3]{\mathrm{CLWE}\Paren{#1,#2,#3}}
\newcommand{\hclwem}[4]{\mathrm{HCLWE}\Paren{#1,#2,#3,#4}}
\newcommand{\nhclwem}[4]{\mathrm{NHCLWE}\Paren{#1,#2,#3,#4}}

\newcommand{\clwe}{\mathrm{CLWE}}

\newcommand{\negl}{\mathrm{negl}}

\newcommand{\TVD}[2]{\mathrm{TVD}(#1,#2)}
\newcommand{\HD}[2]{\mathrm{H}(#1,#2)}

\newcommand{\LTF}{\mathrm{LTF}}
\newcommand{\PTF}{\mathrm{PTF}}

\title{
  Hardness of Agnostically Learning Halfspaces from Worst-Case Lattice Problems\thanks{This project has received funding from the European Research Council (ERC) under the European Union’s Horizon 2020 research and innovation programme (grant agreement No 815464).}
}

\author{
  Stefan Tiegel\thanks{ETH Z\"urich.}
}

\begin{document}

\pagestyle{empty}

\maketitle
\thispagestyle{empty} %

\begin{abstract}

We show hardness of improperly learning halfspaces in the agnostic model, both in the distribution-independent as well as the distribution-specific setting, based on the assumption that worst-case lattice problems, e.g., approximating shortest vectors within polynomial factors, are hard.
In particular, we show that under this assumption there is no efficient algorithm that outputs any binary hypothesis, not necessarily a halfspace, achieving misclassfication error better than $\frac 1 2 - \gamma$ even if the optimal misclassification error is as small is as small as $\delta$.
Here, $\gamma$ can be smaller than the inverse of any polynomial in the dimension and $\delta$ as small as $\exp\Paren{-\Omega\Paren{\log^{1-c}(d)}}$, where $0 < c < 1$ is an arbitrary constant and $d$ is the dimension.
For the distribution-specific setting, we show that if the marginal distribution is standard Gaussian, for any $\beta > 0$ learning halfspaces up to error $\OPT_\LTF + \e$ takes time at least $d^{\tilde{\Omega}(1/\e^{2-\beta})}$ under the same hardness assumptions.
Similarly, we show that learning degree-$\ell$ polynomial threshold functions up to error $\OPT_{\PTF_\ell} + \e$ takes time at least $d^{\tilde{\Omega}(\ell^{2-\beta}/\e^{2-\beta})}$.
$\OPT_\LTF$ and $\OPT_{\PTF_\ell}$ denote the best error achievable by any halfspace or polynomial threshold function, respectively.

Our lower bounds qualitively match algorithmic guarantees and (nearly) recover known lower bounds based on non-worst-case assumptions.
Previously, such hardness results~\cite{D16, diakonikolas2021optimality} were based on average-case complexity assumptions, specifically, variants of Feige's random 3SAT hypothesis, or restricted to the statistical query model.
Our work gives the first hardness results basing these fundamental learning problems on well-understood worst-case complexity assumption.
It is inspired by a sequence of recent works showing hardness of learning well-separated Gaussian mixtures based on worst-case lattice problems.

\end{abstract}

\clearpage

\microtypesetup{protrusion=false}
\tableofcontents{}
\microtypesetup{protrusion=true}

\clearpage

\pagestyle{plain}
\setcounter{page}{1}

\section{Introduction}

An important question in theoretical computer science, and in learning theory in particular, is understanding the relation between average-case and worst-case problems (cf. Levin's work on distributional analogs of NP \cite{levin1986average} and Impagliazzo's five worlds \cite{impagliazzo1995personal}, and also the survey of Bogdanov and Trevisan \cite{bogdanov2006average}).
In particular, to understand for what kind of average-case problems we can show hardness based on worst-case assumptions, thus unlocking the power of the machinery of classical worst-case reductions.
In this work, we make progress on this question by evidencing a strong connection between fundamental and well-studied learning problems and worst-case assumptions with a plethora of other applications.
Specifically, we will show that learning halfspaces and polynomial threshold functions, in either the distribution-independent or distribution-specific setting, are as hard as standard worst-case lattice problems frequently used as a basis of hardness in cryptography \cite{peikert2016decade}.

There are several barriers for basing the hardness of average-case problems on classical assumptions such as $\mathrm{P} \neq \mathrm{NP}$ \cite{applebaum2008basing,DBLP:journals/siamcomp/FeigenbaumF93, bogdanov2006worst} and results in the context of learning theory have either been restricted to the PAC learning setting, in which there is no noise, \cite{kearns1994cryptographic, klivans2009cryptographic}\footnote{We will talk about this a bit more below.}, or restricted to hardness results for (semi-)proper learning, where, loosely speaking, the hypothesis output by the algorithm has to be of the same kind as the one which generated the samples \cite{feldman2006optimal,feldman2006new,DBLP:conf/focs/GuruswamiR06, MR2644358-Gopalan10}.
In fact, there is evidence that this might be inherent \cite{applebaum2008basing}.
On the other hand, there is a plethora of strong hardness results ruling out even \emph{improper} learning algorithms, i.e., that output an arbitrary hypothesis that well-approximates a certain function to be learned, often matching known algorithmic upper bounds.
However, these results can be based only on average-case assumptions \cite{kalai2008agnostically, klivans2014embedding, D16, daniely2021local} or be shown for restricted models of computations \cite{diakonikolas2021optimality}.
So far it remained unclear if these results can also be based on well-understood worst-case assumptions.

In contrast to this, basing hardness of average-case problems on worst-case assumptions is ubiquitous in cryptography and a highly desirable feature.
In particular, many problems are based on worst-case hardness of lattice problems such as the Shortest Independent Vector Problem ($\mathrm{SIVP}$) or the Gap Shortest Vector Problem ($\mathrm{gapSVP}$) (cf. \cref{prob:gap_svp,prob:sivp}).
We do not attempt to survey the vast literature on the topic and instead refer to \cite{peikert2016decade}.
Recent breaktbrough results \cite{CLWE, CLWE_2} have provided a bridge between these lattice problems and learning problems by showing that a certain Gaussian Mixture Model is hard to learn assuming the worst-case hardness of either $\mathrm{SIVP}$ or $\mathrm{gapSVP}$.
In this work we extend this bridge by showing that hardness of other fundamental learning problems can also be based on these assumptions.
Specifically, assuming worst-case hardness of either of the above lattice problems, we show that weak improper learning of halfspaces in the agnostic model is hard.
Further, we extend our results to the setting in which the marginal distribution is fixed to be a standard Gaussian, evidencing that even average-case problems with very specific distributional requirements can be shown to be hard under worst-case assumptions.
This second result also extends to learning polynomial threshold functions.
Precise definitions will follow below.

The task of agnostically learning a class $\cC$ of boolean functions, called a concept class, is defined as follows:
Given samples $(\bm x, y) \in \R^{M} \times \Set{-1,+1}$ from an arbitrary distribution $D$ compute a binary hypothesis $h \colon \R^M \rightarrow \Set{-1,+1}$ achieving small \emph{misclassification error}: $$\err\Paren{h} \coloneqq \Psymb_{(\bm x,y) \sim D}\Paren{h(\bm x) \neq y} \,.$$
In particular, we aim to achieve error close to the minimum misclassification error achieved by any function in $\cC$, denoted by $\OPT_{\cC}$.
We say that $h$ is a \emph{weak learner}, if it achieves error better than $1/2 - 1/\poly(M)$.
Concept classes relevant to this work are the ones of all halfspaces, also known as linear treshold functions (LTFs), defined as $\bm x \mapsto \sign\Paren{\iprod{\bm w, \bm x}}$ for some unknown $\bm w \in \bbS^{M-1}$, and degree-$\ell$ polynomial threshold functions (PTFs), defined as $\bm x \mapsto \sign\Paren{p\Paren{\bm x}}$ for some unknown degree-$\ell$ polynomial $p$.
Note, that we do not restrict the output hypothesis $h$ to belong to $\cC$.
This is called \emph{improper} learning and stands in contrast to so-called proper learning for which most hardness results based on worst-case assumption are known.
In this work we show strong limitations for improperly learning both LTFs and PTFs agnostically under worst-case assumptions.
We remark that if $\OPT_\LTF = 0$, we can efficiently find a halfspace which achieves arbitrarily small misclassfication error \cite{maass1994fast}.
This can be extended to the case when $\OPT_\LTF = O\Paren{\tfrac{\log M} M}$.
Our first result states that even if $\OPT_\LTF$ is just slightly larger, we cannot output any binary hypothesis which achieves error significantly better than a random guess:
\begin{theorem}[Informal version of \cref{thm:agnostic_ltf}]
    \label{thm:main}
    Assuming hardness of either $\mathrm{SIVP}$ or $\mathrm{gapSVP}$, there is no $\poly\Paren{M}$-time algorithm that learns $M$-dimensional halfspaces in the agnostic model up to error $1/2 - 1/\poly\Paren{M}$.
    This holds already if $\OPT_{\mathrm{LTF}}$ is as small as $\exp\Paren{-\log^{1-c}\Paren{M}}$, where $0 < c < 1$ is an absolute constant.
\end{theorem}

Hence, weak improper learning of halfspaces in the agnostic model is likely to be computationally challenging.
It is natural to ask whether the problem becomes easier by making stronger distributional assumptions.
This turns out to indeed be the case.
Specifically, if we restrict to the case that samples $(\bm x, y)$ come from a distribution whose $\bm x$-marginal $D_{\bm x}$ is standard Gaussian, the $\mathrm{L}_1$-regression algorithm from \cite{kalai2008agnostically} is known to learn LTFs up to error $\OPT_{\mathrm{LTF}} + \e$ in time $M^{O\paren{1/\e^2}}$ and degree-$\ell$ PTFs up to error $\OPT_{\mathrm{PTF_\ell}} + \e$ in time $M^{O\paren{\ell^2 / \e^4}}$.
Our second main result shows that under the same assumptions as in \cref{thm:main}, these results are qualitively tight.
\begin{theorem}[Informal version of \cref{thm:distribution_specific}]
    \label{thm:distribution-specific-informal}
    Let $\beta > 0$ be arbitrary and $\e > 0$.
    There exists a distribution $D$ over $\R^M \times \Set{-1,+1}$ such that $D_{\bm x}$ is standard Gaussian and assuming hardness of either $\mathrm{SIVP}$ or $\mathrm{gapSVP}$ there is no $M^{\Omega\Paren{\tfrac 1 {\log(1/\e) \cdot \e^{2-\beta}} }}$-time algorithm which achieves misclassification error $\OPT_\LTF +\e$ over $D$.
    Similarly, there is no $M^{\Omega\Paren{\tfrac {\ell^{2-\beta}} {\log(\ell / \e) \cdot \e^{2-\beta}}}}$-time algorithm which achieves misclassification error $\OPT_{\PTF_\ell} +\e$ over $D$.
\end{theorem}

Our result is inspired by recent hardness results for learning mixtures of well-separated Gaussians \cite{CLWE, CLWE_2} based on the same worst-case lattice problems.
In particular, we show a simple reduction from the \emph{Continuous Learning with Errors} (CLWE) problem introduced in \cite{CLWE}, a continuous analouge of Regev's Learning with Errors problem (LWE) \cite{standard_LWE}.
Indeed, our hard instance in \cref{thm:main} will correspond to a mixture of (a small modification of) two \emph{homogenous} CLWE distributions.
The construction for \cref{thm:distribution-specific-informal} will be similar.
See \cref{sec:techniques} for more details.

\subsection{Relation to Previous Hardness Results}

Our main theorems (almost) match algorithmic upper bounds and (nearly) recover known lower bounds under either average-case hardness assumptions or in restricted models of computation.
In essence, we show that for a class of fundamental learning problems there is no price to pay for basing hardness of learning problems on worst-case assumptions.
Hardness of improperly weakly learning halfspaces in the agnostic model, quantitatively matching the above theorem exactly, was known under a variant of Feige's random 3SAT hypothesis and when assuming $D$ is supported on the boolean hypercube \cite{D16}.
Later a weaker result, that achieving error $\OPT_\LTF + \e$ is hard, was shown under a different assumption on the existence of a certain kind of pseudo-random generators \cite{daniely2021local}.

For the distribution-specific setting, when $D_{\bm x}$ is standard Gaussian, lower bounds were either far from algorithmic guarantees \cite{klivans2014embedding} or only known in the statistical query (SQ) model \cite{kearns1998efficient}.
In particular, it was known that any SQ algorithm achieving error $\OPT_\LTF + \e$ needs at least $2^{M^{\Omega(1)}}$ queries or queries of accuracy at $M^{-\Omega(1/\e^2)}$.
Similarly, any SQ algorithm achieving error $\OPT_{\PTF_\ell} + \e$ needs at least $2^{M^{\Omega(1)}}$ queries or queries of accuracy at $M^{-\Omega(\ell/\e^4)}$ \cite{diakonikolas2021optimality}.
This can be seen as evidence that every algorithm solving the above problems needs time at least $2^{M^{\Omega(1)}}$ or $M^{\Omega(1/\e^2)}$, respecitvely, $M^{\Omega(\ell^2/\e^4)}$, samples.
This (nearly) matches our lower bounds in \cref{thm:distribution-specific-informal}.
We remark that, for learning PTFs, both lower bounds are a $1/\e^2$ factor away from known upper bounds and closing this gap is an interesting open question.
Further, concurrent and independent work \cite{diakonikolas2023near} showed, qualitatively and quantitatively, very similar hardness results for agnostically learning halfspaces under Gaussian marginals.
They also show lower bounds for agnostically learning ReLUs under Gaussian marginals.

\paragraph{Hardness Based on Public-Key Cryptosystems}

We would like to further highlight the connection of our work to two lines of work for proving lower bounds for learning problems.
In a seminal work, Kearns and Valiant pushed forward the idea of basing hardness of learning a concept class $\cC$, specifically when $\OPT_\cC = 0$, on the conjectured security of cryptographic public-key encryption schemes by creating samples for the learning probem by encryption messages oneself  \cite{kearns1994cryptographic}.
They use this to show that improperly learning boolean formulae and deterministic finite automata is hard assuming, e.g., that breaking the RSA cryptosystem is hard.
Later, this approach was used in \cite{klivans2009cryptographic} to show that learning the class of intersections of halfspaces is hard assuming cryptosystems based on LWE are hard \cite{DBLP:conf/stoc/Regev03,DBLP:conf/stoc/Regev05}, which in turn is implied by hardness of either $\mathrm{SIVP}$ or $\mathrm{gapSVP}$.
Again assuming $\OPT_\cC = 0$.
Hence, in the case where the public-key encryption scheme used is hard under worst-case assumptions, also the learning problem enjoys the same hardness guarantees.
However, there are two shortcomings to this approach:
First, we have to find a suitable encryption scheme for a learning problem and additionaly, this scheme has to be hard under worst-case assumptions.
Second, its not clear how to extend this method to the agnostic setting studied in this paper, where $\OPT_\cC > 0$.
Our approach gives a more principled approach for establishing the desired hardness guarantees.

\paragraph{Hardness Based on Learning Parities with Noise}

Secondly, in the past the Learning Parities with Noise (LPN) problem has played a central role in deriving lower bounds for learning problems.
LPN is a special case of LWE whose continuous version we base our lower bounds on.
Crucially however, known worst-case hardness results for LWE do not extend to LPN.
The following hardness results based on LPN are known:
First, \cite{feldman2006new} shows hardness of agnostically learning various boolean functions, not including halfspaces, based on the hardness of a sparse version of LPN - more precisely, that learning parities that depend on only $k$ variables, takes time at least $M^{\Omega(k)}$.
Under the same assumption, \cite{klivans2014embedding} shows that agnostically learning halfspaces under the Gaussian distribution up to error $\OPT_\LTF + \e$ takes time at least $M^{\Omega\Paren{\log(1/\e)}}$.
Second, and more relevant to this work, \cite{kalai2008agnostically} shows that for any $\beta > 0$ an algorithm for agnostically learning halfspaces under the uniform distribution distribution over the hypercube that runs in time $M^{O(1/\e^{2-\beta})}$ implies an algorithm for LPN with constant noise rate running in time roughly $2^{O(M^{1-\beta/2})}$.
While LPN certainly is a central problem in the field of learning theory and all of the above assumptions are widely believed to be true, its worst-case hardness remains poorly understood.
To the best of our knowledge, there is no worst-case hardness result for the sparse version.
The version used by  \cite{kalai2008agnostically}, was recently shown to be hard under a non-standard version of some worst-case assumption\footnote{More specifically, a promise version of the Nearest Codeword Problem with additional assumptions.} \cite{brakerski2019worst, yu2021smoothing}.
Hence, lower bounds based on LPN can only constitute a weak link between fundamental learning problems and worst-case assumptions.
It however is a very interesting question, if this link can be strengthed by basing LPN on more standard worst-case assumptions as is possible for its cousin LWE \cite{DBLP:conf/coco/Regev10}.

\paragraph{Distributions That Are Hard to Distinguish From a Gaussian}

At the core of our results, and more specifically, the CLWE problem (see \cref{sec:techniques} for a definition), lies the fact that a certain distribution is hard to distinguish from the standard Gaussian.
We remark tha this idea is also present in previous lower bound constructions.
In particular, the "parallel pancakes" construction in \cite{DKS17} is the starting point for many lower bounds in the statistical query model \cite{diakonikolas2019efficient,pmlr-v178-diakonikolas22e,DK22, NT22}.
A similar construction was used in \cite{bubeck2019adversarial} to show hardness of a certain binary classification problem in the statistical query model.
Further, CLWE was used in \cite{hardness_periodic_neuron} to show hardness of learning a single periodic neuron.

Lastly, concurrent and independent work \cite{diakonikolas_massart_lwe} shows lower bounds for learning in the so-called Massart model~\cite{MN06} based on LWE and hence also provides a link between learning and worst-case lattice problems.
Previously, such lower bounds were only known in the statistical query model~\cite{CKMY20,DK22,NT22}.

\section{Technical Overview}
\label{sec:techniques}

\paragraph{Continuous Learning with Errors}

Before we start describing our lower bound constructions, we introduce the continuous learning with errors (CLWE) problem.
Let $\bm w$ be uniform over the unit sphere, $\bm y \sim N(0, I)$, and $\gamma, \beta > 0$ be some parameters.
We are given samples $(\bm y, z)$, where $$z = \gamma \iprod{\bm w, \bm y} + e \mod 1 \,,$$ for $e \sim N(0,\beta^2)$\footnote{For ease of notation we have slightly rescaled the problem. See \cref{def:clwe_distribution} for the exact definition we use.}.
The task is to distinguish these samples from samples $(\bm y, z)$, where $\bm y \sim N(0,I)$ as well, but $z$ is independently and uniformly at random drawn from $[0,1)$,\footnote{This is called the \emph{decision version}. In the \emph{search} version one asks instead to recover the hidden direction $\bm w$.}
For convenience, we call this second distribution $\mathrm{CLWE}^{\mathrm{null}}$.
\cite{CLWE} gave a (quantum) reduction from approximating the Gap Shortest Vector Problem ($\mathrm{GapSVP}$) or the Shortest Independent Vector Problem ($\mathrm{SIVP}$) within polynomial factors to CLWE.
In \cite{CLWE_2} this was strengthened, for some set of parameters, to a reduction directly from standard LWE implying hardness also when only assuming the classical hardness of the above lattice problems.
Both works use the CLWE problem to obtain hardness results for density estimation of well-separated mixtures of Gaussians.
As remarked earlier, the idea of desigining a distribution that is hard to distinguish earlier also lies at the heart of many statistical query lower bounds.
See e.g. the influential work \cite{DBLP:conf/focs/DiakonikolasKS17} and subsequent works.

\paragraph{Distribution-Independent Setting}

We next give a sketch of the proof of \cref{thm:main}.
First, it is clear that in order to show lower bounds for learning halfspaces, it is enough to show lower bounds for learning polynomial threshold functions over a lower-dimensional space.
More specifically, let $M,n,\ell \in \N$ be such that $M = \binom{n+\ell}{n} \leq n^\ell$, then any degree-$\ell$ PTF over $\R^n$ can be viewed as a halfspace over $\R^M$ by using an embedding that maps $\bm x$ to the vector containing all monomials of degree at most $\ell$.\footnote{This is sometimes referred to as the Veronese mapping, or a feature map.}
In what follows we will choose parameters such that $n \approx \log\Paren{M}^{1+c}$ for some constant $c > 0$.
Hence, to rule out polynomial-time algorithms, in $M$, for learning halfspaces over $\R^M$ it is enough to show an exponential lower bound, in $n$, for learning degree-$\ell$ PTFs over $\R^n$.

There are two parts to showing \cref{thm:main}.
We aim to find a distribution $D$ such that:
First, in sub-exponential time we cannot compute a binary hypothesis that has misclassfication error significantly better than $1/2$ on $D$ and second, there exists a degree-$\ell$ PTF such that $\OPT_{\PTF_\ell}$ is vanishing.
By the discussion above this implies that $\OPT_\LTF$ is vanishing as well.
We will choose $D$ to correspond to a mixture of variants of the CLWE distribution.
In what follows we set $\gamma \geq 2 \sqrt{n}$ and $\beta = 1/\poly(n)$.
\cite{CLWE,CLWE_2} show that for this choice of parameters there is no sub-exponential time algorithm for distinguishing such samples from $N(0,I_n) \times \cU([0,1))$ assuming that there is no, quantum or classcial, respectively, sub-exponential time algorithm for $\mathrm{GapSVP}$ and $\mathrm{SIVP}$.\footnote{In the hardness result of \cite{CLWE_2}, $\bm w$ is not a random unit vector but rather a random sparse unit vector.}

Moreover, they introduced a variant of the CLWE distribution, which intuitively can be thought of as the CLWE distribution conditioned on $z \approx 0$.
This is called the \emph{homogeneous} CLWE (short hCLWE) distribution (cf. \cref{def:hclwe_distribution}) and will be the basis of our hardness result.
They show that it is equal to an infinite mixture of Gaussians and has density roughly proportional to $$\sum_{k \in \Z} N\Paren{0, \gamma^2}(k) \cdot N\Paren{0, I_n - \bm w {\bm w}^\top}\paren{\pi_{{\bm w}^\perp}(\bm y)} \cdot N\Paren{\tfrac k \gamma, \tfrac {\beta^2}{\gamma^2}} \Paren{\iprod{\bm w, \bm y}}\,,$$ where $N(\mu,\Sigma)(\bm x)$ denotes the density of $N(\mu,\Sigma)$ evaluated at $\bm x$ and $\pi_{{\bm w}^\perp}(\bm y)$ the projection of $\bm y$ onto the space orthogonal to $\bm w$.
Note that the components are equally spaced along direction $\bm w$ with spacing $1/\gamma$ and the $k$-th component has weight roughly $\exp\Paren{-k^2/\gamma^2}$.
Second, along the direction of $\bm w$ they have variance $\approx \beta/\gamma \ll 1/\gamma$, i.e., they are almost non-overlapping, and in all other directions have variance 1.
The authors show that under the same hardness assumption, there is no sub-exponential time algorithm that can distinguish the hCLWE distribution from the standard Gaussian.

In particular, let $H_0$ be the hCLWE distribution.
Additionally, let $H_{1/2}$ be obtained in the same way but instead of conditioning on $z \approx 0$ we condition on $z \approx 1/2$.
The resulting distribution will be the same as $H_0$ but the components are shifted along the direction $\bm w$ by $1/(2\gamma)$.
Further, it enjoys the same hardness guarantees as $H_0$.
Since $\beta \ll \gamma$ the two distributions will only overlap in a region of exponentially small probability mass.
In fact, if we consider the distributions $H_0'$ and $H_{1/2}'$ in which each component of the mixture is truncated such that they are completely disjoint (by some small margin) this only introduces a negligible change in total variation distance.
It follows by a standard argument (cf. \cref{lem:small_tvd_advantage}), that $H_0'$ and $H_{1/2}'$ will still be hard to distinguish from a standard Gaussian.
\cite{CLWE} showed how to obtain samples from $H_0$ using CLWE samples and their argument straigtforwardly extends to obtaining samples from $H_1$.
Hence, we can also obtain samples from the mixture distribution over $\R^n\times\Set{-1,+1}$ defined as $$D = \frac 1 2 \cdot \Paren{H_0, +1} + \frac 1 2 \cdot \Paren{H_{1/2}, -1}$$
by deciding for each sample whether it should be generated from $H_0$ or $H_1$ with probability $1/2$ and setting the label accordingly.
Applying this same procedure to samples from $\mathrm{CLWE}^{\mathrm{null}}$, we can see that $D$ is hard to distinguish from $D^{\mathrm{null}}_n \coloneqq N\Paren{0,I_n} \times \mathrm{Be}\Paren{\tfrac 1 2}$, where $\mathrm{Be}\Paren{\tfrac 1 2}$ denotes the distribution that is $-1$ with probability $1/2$ and $+1$ with probability $1/2$.
Again, we can instead consider the distribution $$D' = \frac 1 2 \cdot \Paren{H_0', +1} + \frac 1 2 \cdot \Paren{H_{1/2}', -1}\,.$$

First, notice that since any learning algorithm has error $1/2$ on $D^{\mathrm{null}}_n$ it follows that we cannot compute, in sub-exponential time, a hypothesis with misclassification error significantly better on $D'$ either since otherwise we could distinguish the two distributions.
Now that we have established that $D'$ is hard to learn, to show our hardness result, we need to show that there is indeed a PTF which achieves vanishing error.
First, note that we can restrict our attention to the direction $\bm w$ by considering a one-dimensional polynomial $p_{\bm w} \colon \R \rightarrow \R$ and then obtaining the final polynomial $p \colon \R^n \rightarrow \R$ as $p\Paren{\bm x} = p_{\bm w} \Paren{\iprod{\bm w, \bm x}}$.
Consider the union of intervals $$S_+ = \bigcup_{k \in \Z} \,\Brac{\tfrac k \gamma - \alpha,\tfrac k \gamma + \alpha}\,,\quad\quad S_- = \bigcup_{k \in \Z} \,\Brac{\tfrac k \gamma + \tfrac 1 {2\gamma} - \alpha,\tfrac k \gamma + \tfrac 1 {2\gamma} + \alpha}\,,$$ where $\alpha < 1/(2\gamma)$ is the radius around which we truncate the components.
Note that by construction the supports of $H_0'$ and $H_{1/2}'$ are equal to $$\supp\Paren{H_0'} = \Set{\bm x \suchthat \iprod{\bm w, \bm x} \in S_+}\,,\quad\quad \supp\Paren{H_{1/2}'} = \Set{\bm x \suchthat \iprod{\bm w, \bm x} \in S_-}\,.$$

Further, let $$S_+^{(\ell)} = \bigcup_{k =-\ell}^\ell \,\Brac{\tfrac k \gamma - \alpha,\tfrac k \gamma + \alpha}\,,\quad\quad S_-^{(\ell)} = \bigcup_{k = -\ell}^{\ell-1} \,\Brac{\tfrac k \gamma + \tfrac 1 {2\gamma} - \alpha,\tfrac k \gamma + \tfrac 1 {2\gamma} + \alpha}\,.$$
Consider the degree-$4\ell$ polynomial $p_{\bm w}$ that has is positive on $S_+^{(\ell)}$ and negative on $S_-^{(\ell)}$ and positive for points of magnitude larger than those in $S_+^{(\ell)} \cup S_-^{(\ell)}$.
By choosing it such that its roots are halfway betwen the intervals we will have some small margin.
Clearly, for $(\bm x, y) \sim D'$ such that $\iprod{\bm w, \bm x} \in S_+^{(\ell)} \cup S_-^{(\ell)}$ we have $y = \sign\Paren{p\Paren{\bm x}}$ always.
The same holds for $(\bm x, y)$ such that $\iprod{\bm w, \bm x} \in S_+ \setminus S_+^{(\ell)}$.
On the flipside, we note that for $(\bm x, y)$ such that $\iprod{\bm w, \bm x} \in S_- \setminus S_-^{(\ell)}$ we have $$-1 = y \neq \sign\Paren{p\Paren{\bm x}} = 1$$ always.
Hence, the total misclassfication error is equal to the probability that $\iprod{\bm w, \bm x} \in S_- \setminus S_-^{(\ell)}$,
This happens if and only if $\bm x$ comes from $H_{1/2}'$ and in particular from a component that doesn't belong to the $2\ell$ most central ones.
Since the $k$-th component has weight $\approx \exp\Paren{-k^2/\gamma^2}$ it follows that this event happens with probability roughly $\exp\Paren{-\ell^2/\gamma^2}$.
For our choice of parameters we have $\gamma = 2\sqrt{n} \approx \log^{(1+c)/2}\Paren{M}$ and $\ell \approx \log\Paren{M}$ and hence the error of $p$ becomes $$\exp\Paren{-\ell^2/\gamma^2} = \exp\Paren{-\log^{(1-c)}\Paren{M}}$$ as desired.

\paragraph{Distribution-Specific Setting}

For the distribution-specific setting (cf. \cref{thm:distribution-specific-informal}), we have the additional requirement that the marginal distribution needs to be standard Gaussian.
Note that this implies that the above lifting to PTFs no longer works:
Indeed, it even is unclear how the distribution before the lifting should look like so that it is standard Gaussian afterwards.
Hence, we work directly with the CLWE problem in dimension $M$.
Recall that this means that $\gamma = 2\sqrt{M}$ and $\beta = 1/\poly(M)$.
This time, to preserve the marginal distribution, let $H_0$ be obtained by conditioning the CLWE distribution on $z \in [0,1/2)$ and $H_1$ by conditioning on $z \in [1/2,1)$.
Our hard distribution will be $$D = \frac 1 2 \cdot (H_0, +1) + \frac 1 2 \cdot (H_1, -1) \,.$$
Note that since $[0,1/2)$ and $[1/2,1)$ partition $[0,1)$, it follows that the marginal of $D$ is the same as the marginal distribution of $y$ in CLWE, i.e., standard Gaussian.
Note that, given CLWE samples, we can obtain samples from $D$ by rejection sampling.
If we apply the same rejection sampling procedure to samples from $\mathrm{CLWE}^{\mathrm{null}}$ we obtain samples from $D^{\mathrm{null}}_M \coloneqq N(0,I_M)  \times \mathrm{Be}\Paren{\tfrac 1 2}$.
Hence, a sub-exponential, in $M$, algorithm to distinguish $D$ and $D^{\mathrm{null}}_M$ with non-negligible advantage can be used to distinguish samples from CLWE and $\mathrm{CLWE}^{\mathrm{null}}$.

It remains to show that if we could learn LTFs and PTFs over $D$ up to error better than $\OPT + \e$ we can distinguish $D$ from $D^{\mathrm{null}}_M$.
For this, we first inspect $D$ more closely.
As for the distribution-indepedent setting, the label of samples from $D$ only depends on the direction $\bm w$.
Second, let $A_k = [\tfrac k \gamma,\tfrac {k+1/2} \gamma), B_k = [\tfrac {k+1/2} \gamma ,\tfrac {k+1} \gamma)$ and 
\[
    S_+ = \bigcup_{k \in \Z} \,A_k\,,\quad\quad S_- = \bigcup_{k \in \Z} \,B_k\,.
\]
It turns out that $D$ is sufficiently well approximated (cf. \cref{lem:tvd_distribution_specific}) by the distribution $D'$ whose marginal is standard Gaussian and for wich it holds that $y = 1$ if and only if $\iprod{\bm w, \bm x} \in S_+$.
More specifically, the total variation distance between $D$ and $D'$ is at most $1/\poly(M)$ and hence affects the missclassification error by at most this same additive factor.
We hence continue to work with $D'$ below.
Regarding LTFs, consider the function $f$ defined as $\bm x \mapsto \sign\Paren{\iprod{\bm w, \bm x}}$.
For simplicity, denote $z = \iprod{\bm w, \bm x}$.
Clearly, this function only misclassfies samples for which either $z \geq 0$ and $z \in S_-$ or $z \leq 0$ and $z \in S_+$.
Let $X \sim N(0,1)$.
By symmetry it follows that
\[
    \err_{D'}\Paren{f} = 2 \Psymb\Paren{z \geq 0 \,, z \in S_-} = 2 \sum_{k \geq 0} \Psymb\Paren{X \in B_k} \,.
\]
Notice that for $k \geq 0$, $\Psymb\Paren{X \in B_k} \leq \Psymb\Paren{X \in A_k}$ always since the pdf of a one-dimensional Gaussian is decreasing for $z \geq 0$.
Further, one can show (cf. \cref{lem:distribution_specific_optimal_ltf}) that for $k \geq \gamma$ is it decreasing sufficiently fast such that that $$2\Psymb\Paren{X \in B_k} \leq \Paren{1-\frac 1 \gamma} \cdot \Brac{\Psymb\Paren{X \in A_k} +\Psymb\Paren{X \in B_k}} \,.$$
Hence, we obtain that there exists an absolute constant $c > 0$ such that
\begin{align*}
    \err_{D'}\Paren{f} &\leq \sum_{0 \leq k < \gamma} \Psymb\Paren{X \in A_k} +\Psymb\Paren{X \in B_k} + \Paren{1 - \frac 1 \gamma} \cdot \sum_{k \geq \gamma} \Psymb\Paren{X \in A_k} +\Psymb\Paren{X \in B_k}\\
    &= \frac 1 2 - \frac 1 \gamma \cdot \Psymb\Paren{X \geq 1} = \frac{1}{2} - \frac{c}{\sqrt{M}} \,.
\end{align*}
Hence, for arbitrary $\beta > 0$, an algorithm achieving misclassfication error $\OPT_\LTF + \e$ for $\e \approx 1/\sqrt{M}$ necessarily needs time at least $2^{\Omega\Paren{M^{1-\beta}}} = M^{\Omega\Paren{\tfrac 1 {\e^{2-\beta} \cdot \log(1/\e) }}}$.

Our argument for degree-$\ell$ PTFs will be similar.
For simplicity, assume that $\ell$ is even and consider the one-dimensional polynomial $p$ defined as follows:
It has roots $-\tfrac \ell {2 \gamma}, - \tfrac {\ell-1} {2\gamma}, \ldots, 0, \ldots, \tfrac {\ell - 1} {2\gamma}, \tfrac \ell {2\gamma}$ and its sign is positive between 0 and $\tfrac 1 {2\gamma}$ and alternates on the other intervals.
For simplicity, also assume without loss of generality that it has positive sign for $z \geq \tfrac \ell {2\gamma}$.
We define the polynomial threshold function $h$ as $\bm x \mapsto \sign\Paren{p\Paren{\iprod{\bm w, \bm x}}}$.
Let again $X \sim N(0,1)$, by symmetry and using the results above it follows that there exists an absolute constant $c > 0$ such that 
\begin{align*}
    \err_{D'}\Paren{h} &= 2 \sum_{k \geq \ell/2} \Psymb\Paren{X \in B_k} = 2 \sum_{k \geq 0} \Psymb\Paren{X \in B_k} - 2 \sum_{k < \ell/2} \Psymb\Paren{X \in B_k} \leq \frac 1 2 - \frac c \gamma - \Psymb\Paren{\frac 1 \gamma \leq X \leq \frac{\ell/2 + 1}{\gamma}} \\
    &\leq \frac 1 2 - \frac c \gamma - \frac 1 2 \cdot \Psymb\Paren{0 \leq X \leq \frac{\ell/2 + 1}{\gamma}} \,.
\end{align*}
Since $\ell \ll \gamma$ the pdf of the standard Gaussian is roughly constant between 0 and $\tfrac{\ell/2 + 1}{\gamma}$.
Hence, it follows that there exsist an absolute constant $c' > 0$ such that $\err_{D'}\Paren{h} \leq \tfrac 1 2 - \tfrac{c' \ell} \gamma$.
It follows as for LTFs, that, for arbitrary $\beta > 0$, an algorithm achieving misclassfication error $\OPT_{\PTF_\ell} + \e$ for $\e \approx \ell/\sqrt{M}$ necessarily needs time at least $2^{\Omega\Paren{M^{1-\beta}}} = M^{\Omega\Paren{\tfrac {\ell^{2-\beta}} {\e^{2-\beta} \cdot \log(\ell/\e) }}}$.

We remark that in both the LTF as well as the PTF case, $\OPT$ is very close to 1/2.
Indeed, this is a property shared with all known lower bounds irrespective of the hardness assumptions/model of computation \cite{kalai2008agnostically, klivans2014embedding, diakonikolas2021optimality}.
It would be very desirable to show lower bounds where this is not the case, as for the distribution-independent setting.\footnote{Concurrent and independent work \cite{diakonikolas2023near} shows very similar hardness results for agnostically learning halfspaces under Gaussian marginals. In particular, their time complexity lower bounds are quantitatively very close to ours, but they additionally can allow for $\e$ as large as roughly $1/\sqrt{\log d}$. The authors use a reduction from CLWE similar to the one presented in this paper. However, they also prove a strengthening of the reductions of \cite{CLWE_2} from LWE to  CLWE. This ultimately leads to hard instances for wider ranges of $\e$. It seems plausible, that one could also combine the more efficient reduction from LWE to CLWE proposed by \cite{diakonikolas2023near} with the reduction from CLWE to agnostically learning halfspaces with Gaussian marginals we presented in this paper to allow for a similar range of $\e$ as in \cite{diakonikolas2023near}}
\section{Preliminaries}
\label{sec:preliminaries}

\subsection*{Notation}

We use boldfont for vectors and non-boldfont for scalars.
We denote $\R_{\geq 0} = [0,\infty)$ and $\R_{> 0} = (0,\infty)$.
For a set $S$, we denote by $\cU(S)$ the uniform distribution over $S$.
We define the Total Variation Distance between two measures $P$ and $Q$ as $$\TVD{P}{Q} = \sup_A \Abs{P(A) - Q(A)}\,.$$

Let $n$ be some parameter.
For the problem of distinguishing two distributions $D_n^{0}$ and $D_n^1$ we define the advantage of an algorithm $\cA$ as $$\Abs{\Psymb_{x \sim D_n^0} \Paren{\cA(x) = 0} - \Psymb_{x \sim D_n^1} \Paren{\cA(x) = 0}}\,.$$
We say that an algorithm has non-negligible advantage if it has advantage $\Omega(n^{-c})$ for some constant $c > 0$.

Let $p \in [0,1/2]$.
We denote by $\mathrm{Be}(p)$ the distribution that is equal to +1 with probability $p$ and equal to -1 with probability $1-p$.

Let $\cX$ be some set and $D$ be a distribution over $\cX \times \Set{-1,+1}$.
Further, let $h \colon \cX \rightarrow \Set{-1,+1}$ be a binary hypothesis.
We denote the \emph{misclassification error} of $h$ as $$\err_D \Paren{h} = \Psymb_{(x,y) \sim D}\Paren{h(x) \neq y} \,.$$
Most of the time the distribution $D$ will be clear from context and we will omit the subscript.
We denote by $D_{\bm x}$ the marginal distribution of $D$ over $\cX$.

\subsection*{Gaussian Distributions}

We denote the standard $n$-dimensional Gaussian distribution by $N(0,I_n)$.
If the dimension is clear from context, we sometimes drop the subscript of the identity matrix.
For $s > 0$, we denote by $\rho_s \colon \R^n \rightarrow \R_+$ the function $$\rho_s(\bm x) = \exp(-\pi \snorm{\bm x /s})\,.$$
If $s = 1$, we omit the subscript.
Note that $\rho_s/s^n$ is equal to the probability density function of the $n$-dimensional Gaussian distribution with mean 0 and covariance matrix $s^2/(2\pi) \cdot I_n$.
In particular, it holds that $$\int_{\R^n} \rho_s(\bm x) \,d\bm x = s^n \,.$$
We define $\rho_s(\bm x \,; \bm c) = \rho_s(\bm x - \bm c)$ and for $\alpha > 0$ we define
$$\rho_s^{\alpha}(\bm x \,; \bm c) = \begin{cases}
    \tfrac 1 Z \cdot \rho_s(\bm x \,; \bm c) \,, &\quad \text{if } \Norm{\bm x - \bm c} \leq \alpha \,, \\
    0 \,, &\quad \text{otherwise,}
\end{cases}$$
where $$Z = \frac{\int_{\Norm{\bm x - \bm c} \leq \alpha} \rho_s(\bm x \,; \bm c) \, d \bm x}{\int_{\R} \rho_s(\bm x \,; \bm c) \, d \bm x}\,.$$

For a lattice $L \sse \R^n$ and $s > 0$ we define the discrete Gaussian distribution $D_{L, s}$ with width $s$ as having support $L$ and probability mass proportional to $\rho_s$.
Further, for a discrete set $S$, we define $\rho_s\Paren{S} = \sum_{x \in S} \rho_s\Paren{x}$.

\subsection*{Various Other Distributions}

\begin{definition}[CLWE Distribution]
    \label{def:clwe_distribution}
    Let $\bm w \in \R^n$ be a unit vector and $\beta, \gamma > 0$.
    Define the distribution $\clwed{\bm w}{\beta}{\gamma}$ over $\R^n \times [0,1)$ as follows.
    Draw $\bm y \sim N(0,\tfrac 1 {2\pi} \cdot I_n)$, $e \sim N(0,\beta^2/(2\pi))$ and let $$z = \gamma \iprod{\bm w, \bm y} + e \mod 1\,.$$
    Note that the density of this distribution is given by $$p(\bm y, z) = \frac{1}{\beta} \cdot \rho \Paren{\bm y} \cdot \sum_{k \in \Z} \rho_\beta \Paren{z + k - \gamma \iprod{\bm w, \bm y}} \,.$$

    Further, let $m \in\N$.
    We denote by $\clwem{m}{\gamma}{\beta}$ the distribution obtained by first drawing $\bm w \sim \cU(\cS^{n-1})$ and then drawing $m$ independent samples from $\clwed{\bm w}{\gamma}{\beta}$.
\end{definition}

\begin{definition}[Homogeneous CLWE (hCLWE) Distribution]
    \label{def:hclwe_distribution}
    Let $\bm w \in \R^n$ be a unit vector, $c \in [0,1)$, and $\beta, \gamma > 0$.
    Let $\pi_{{\bm w}^\perp}(\bm y)$ be the projection of $\bm y$ onto the space orthogonal to $\bm w$.
    Define the distribution $\hclwed{\bm w}{\beta}{\gamma}{c}$ over $\R^n$ as having density at $\bm y$ proportional to
    \begin{align}
        \sum_{k \in \Z} \rho_{\sqrt{\beta^2 + \gamma^2}}(k \,; c) \cdot \rho\Paren{\pi_{{\bm w}^\perp}(\bm y)} \cdot \rho_{\beta / \sqrt{\beta^2 + \gamma^2}}\Paren{\iprod{\bm w, \bm y} \,; \frac{\gamma}{\beta^2 + \gamma^2} (k-c) }\,. \label{eq:hclwe}
    \end{align}

    Further, let $m \in\N$.
    We denote by $\hclwem{m}{\gamma}{\beta}{c}$ the distribution obtained by first drawing $\bm w \sim \cU(\cS^{n-1})$ and then drawing $m$ independent samples from $\hclwed{\bm w}{\gamma}{\beta}{c}$.
\end{definition}

Note that~\cref{eq:hclwe} integrates to $Z =\tfrac{\beta}{\sqrt{\beta^2 + \gamma^2}} \cdot \rho_{\sqrt{\beta^2 + \gamma^2}}\Paren{\Z\,;c}$.
Further, \cref{eq:hclwe} is equivalent to (see \cref{fact:eq_hclwe})
\begin{align}
    \rho(\bm y) \cdot \sum_{k \in \Z} \rho_{\beta} \Paren{\gamma \iprod{\bm w, \bm y} \,;  k - c}\,. 
\end{align}

Intuitively, one can think of the $\hclwed{\bm w}{\gamma}{\beta}{c}$ distribution as $\clwed{\bm w}{\gamma}{\beta}$ conditioned on $z = c$.

\begin{definition}[Non-Overlapping hCLWE Distribution]
    Let $\bm w \in \R^n$ be a unit vector, $c \in [0,1), \beta, \gamma > 0$ and $\alpha = \frac 1 {10} \cdot \frac{\gamma}{\gamma^2 + \beta^2}$.
    Define the distribution $\nhclwed{\bm w}{\beta}{\gamma}{c}$ over $\R^n$ as having density proportional to
    \begin{align}
        \sum_{k \in \Z} \rho_{\sqrt{\beta^2 + \gamma^2}}(k\,;c) \cdot \rho\Paren{\pi_{{\bm w}^\perp}(\bm y)} \cdot \rho_{\beta / \sqrt{\beta^2 + \gamma^2}}^\alpha\Paren{\iprod{\bm w, \bm y} \,; \frac{\gamma}{\beta^2 + \gamma^2}\Paren{k - c} }\,. \label{eq:nhclwe}
    \end{align}

    Further, let $m \in\N$ and $\cS$ be a distribution over unit vectors in $\R^n$.
    We denote by $\nhclwem{m}{\gamma}{\beta}{c}$ the distribution obtained by first drawing $\bm w \sim \cU(\cS^{n-1})$ and then drawing $m$ independent samples from $\nhclwed{\bm w}{\gamma}{\beta}{c}$.
\end{definition}

Note that this is the same as the hCLWE distribution but with the individual components of the mixture truncated in the hidden direction.
By definition of $\rho^\alpha$ \cref{eq:hclwe,eq:nhclwe} integrate to the same value.
$\alpha$ is chosen such that the components become non-overlapping but the resulting distribution has small total variation distance to the corresponding non-truncated hCLWE distribution.
Although this is strictly speaking not necessary to prove our result, we will see that having non-overlapping components will simplify our analysis.

\subsection*{Hardness Assumption}

We make the following hardness assumption
\begin{assumption}
    \label{assump:disting_clwe}
    Let $n, m \in \N$ and $$\gamma \geq 2\sqrt{n}\,,\quad\quad \beta = \frac 1 {\poly\Paren{n}}\,.$$
    Further, let $\delta < 1$ be arbitrary and $m = 2^{n^\delta}$.
    There is no $2^{n^\delta}$-time distinguisher between $$\clwem{m}{\gamma}{\beta} \quad\text{and}\quad N\Paren{0, \tfrac 1 {2\pi} \cdot I_n}^m \times U\Paren{[0,1)}^m$$ with non-negligible advantage.
\end{assumption}

Note that by \cite[Corollary 3.2]{CLWE} this is implied by assuming quantum hardness of approximating either the Shortest Independent Vector Problem or the Gap Shortest Vector Problem withing polynomial factors.
For completeness, we define the problems explicitly below.
For more brackground on these, we refer to \cite{peikert2016decade}.
An $n$-dimensional \emph{lattice} $L$ is defined to be a discrete additive subgroup of $\R^n$.
It can be fully specified by a basis $B \in \R^{n \times n}$ as $L = B \Z^{n}$.
We will only consider the case in which $B$ is full-rank.
For $1\leq i \leq n$, consider 
\[
    \lambda_i\Paren{L} \coloneqq \inf\Set{r > 0 \suchthat \dim\Paren{\Span\Paren{L \cap B_r(0)} \geq i}} \,.
\]
We can now define $\mathrm{GapSVP}$ and $\mathrm{SIVP}$.
\begin{problem}[Gap Shortest Vector Problem ($\mathrm{GapSVP}$)]
    \label{prob:gap_svp}
    Let $\alpha = \poly(n)$ be arbitrary.
    Given an $n$-dimensional lattice $L$ and $d > 0$ such that either (a) $\lambda_1\Paren{L} \leq d$ or (b) $\lambda_1\Paren{L} > \alpha \cdot d$, decide whether (a) or (b) holds.
\end{problem}

\begin{problem}[Shortest Independent Vector Problem ($\mathrm{SIVP}$)]
    \label{prob:sivp}
    Let $\alpha = \poly(n)$ be arbitrary.
    Given an $n$-dimensional lattice $L$ output a set of linearly independent lattice points of length at most $\alpha \cdot \lambda_n\Paren{L}$.
\end{problem}

\section{Hardness of Distribution-Independent Learning}
\label{sec:hardness_agnostic}

In this section we are going to prove a formal version of \cref{thm:main}.
In particular, we will show the following theorem
\begin{theorem}
    \label{thm:agnostic_ltf}
    Let $M \in \N$ and $0 < c < c' <  1$ be arbitrary.
    There exists a distribution $D$ over $\R^M \times \Set{-1,+1}$ such that under \cref{assump:disting_clwe} there is no algorithm using fewer than $\exp\Paren{\Omega\Paren{\log^{1+c}\Paren{M}}}$ samples and running in time $\exp\Paren{\Omega\Paren{\log^{1+c}\Paren{M}}}$ that outputs any binary hypothesis $f$ such that $$\err_D\Paren{f} \leq \frac 1 2 - \exp\Paren{-\Omega\Paren{\log^{1+c}\Paren{M}}} \,.$$

    This holds even if there exists a linear threshold function $f^*$ such that $$\err_D\Paren{f^*} \leq \exp\Paren{-\Omega\Paren{\log^{1-c'} \Paren{M}}}$$ and for all $x \in \R^M$ in the support of $D$ it holds that $$\Psymb_{(\bm x,y) \sim D} \Paren{f^*(x) \neq y \suchthat \bm x} \in \Set{0,1} \,.$$
\end{theorem}

We will show hardness by showing that a certain low-degree polynomial threshold function is hard to learn.
Hardness of learning halfspaces then follows by embedding this into a higher-dimensional space.
Note that the last two properties of the distribution imply that an overwhelming fraction of the observed points is in fact noiseless.
More concretely, we will use the following lemma.
We provide in a proof in \cref{sec:missing_lemmas} for completeness.

\begin{lemma}
    \torestate{
    \label{lem:hard_ptf_implies_hard_ltf}
    Let $n,d \in \N$ and $M \geq n^d$.
    Further, let $D$ be a distribution over $\R^n \times \Set{-1,+1}$.
    There exists a distribution $D'$ over $\R^M \times \Set{-1,+1}$ such that
    \begin{enumerate}
        \item For every degree-$d$ polynomial threshold function $h \colon \R^n \rightarrow \Set{-1,+1}$ there exists a linear threshold function $f \colon \R^M \rightarrow \Set{-1,+1}$ such that $$\err_{D'}\Paren{f'} = \err_D\Paren{h}\,.$$
        \item For every binary function $f \colon \supp\Paren{D'} \rightarrow \Set{-1,+1}$ there exists a binary function $h \colon \R^n \rightarrow \Set{-1,+1}$ such that $$\err_{D'}\Paren{f'} = \err_D\Paren{h}\,.$$
    \end{enumerate}
    In both cases such a function can be computed in time $\poly\Paren{M}$.
    Moreover, there exists a one-to-one mapping $\phi \colon \supp\Paren{D} \rightarrow \supp\Paren{D'}$ such that in both of the above cases for all $\bm {\tilde{x}'} \in \supp(D')$ it holds that $$\Psymb_{(\bm x',y') \sim D'} \Paren{f(\bm x') \neq y' \suchthat \bm x' = \bm {\tilde{x}'}} = \Psymb_{(\bm x,y) \sim D} \Paren{h(\bm x) \neq y \suchthat \bm x = \phi^{-1}(\bm {\tilde{x}'})} \,.$$
    }
\end{lemma}

The hard distribution will correspond to a mixture of two non-overlapping hCLWE instances for an appropriate choice of parameters.
More precisely, we will use the following lemma
\begin{lemma}
    \label{lem:nhclwe_is_agnostic_ptf}
    Let $d, n \in \N, \beta, \gamma \in \R_{> 0}$ such that $$\beta^2 \leq \gamma^2\,, \quad\quad  \frac d \gamma = \Omega(1) \,.$$
    Further, let $c_+ = 0, c_- = 1/2$, and $\bm w \in \mathbb{S}^{n-1}$.
    Let $$D_+ = \nhclwed{\bm w}{\beta}{\gamma}{c_+}\,, \quad\quad D_- = \nhclwed{\bm w}{\beta}{\gamma}{c_-}\,.$$
    Let $\cC_{4d}$ be the class of all degree-$4d$ polynomial threshold functions (PTFs).
    Consider the distribution over $\R^n \times \Set{-1,+1}$ given by $$D = \frac 1 2 \cdot \Paren{D_+, +1} + \frac 1 2 \cdot \Paren{D_-, -1} \,.$$
    There exists a degree-$4d$ PTF $h^*$ such that $$\err\Paren{h^*} \leq \exp\Paren{- \Omega \Paren{\frac{d^2}{\gamma^2}}} \,.$$
    Moreover, it holds that $$\forall \bm x \in \R^n \colon \quad \Psymb_{\paren{\bm x,y} \sim D} \Paren{h^*(\bm x) \neq y \given \bm x} \in \Set{0,1} \,.$$
    
\end{lemma}

With this in hand, we continue with the proof of \cref{thm:agnostic_ltf}
\begin{proof}[Proof of \cref{thm:agnostic_ltf}]
    Let $1 > c' > c'' > c > 0, \tfrac{1+c}{1+c''} < \delta < 1$ and $$d = \left\lceil \frac 1 4 \cdot \frac \delta {1+c} \cdot \frac{\log M}{\log \log M} \right\rceil\,,$$ where $C$ is a large enough universal constant.
    Further, let $n$ be the largest natural number such that $n^{4d} \leq M$.
    In what follows, we will for simplicity assume that $n^{4d} = M$, all arguments can readily be adapted to the general case.
    We will show that there exists a distribution $D$ over $\R^n \times \Set{-1,+1}$ such that under \cref{assump:disting_clwe} there is no algorithm using fewer than $\exp\Paren{n^\delta}$ samples and running in time at most $\exp\Paren{n^\delta}$ that outputs any binary hypothesis achieving misclassification error better than $1/2 - \tau$, for $$\tau = \exp\Paren{-c_\tau \cdot n^\delta}\,,$$ for a small enough absolute constant $c_\tau$.
    Note that this implies the first part of the theorem since
    \begin{align}
        n^\delta = \exp\Paren{\frac \delta {4d} \cdot \log M} = \exp\Paren{\Paren{1+c} \cdot \log \log M + \Theta\Paren{1}} = \Theta\Paren{\log^{1+c}M} \label{eq:def_n}
    \end{align}
    and hence $$\exp \Paren{n^\delta} = \exp\Paren{O\Paren{\log^{1+c}\Paren{M}}} \quad\quad\text{and}\quad\quad \tau = \exp\Paren{-\Omega\Paren{\log^{1+c}\Paren{M}}} \,.$$
    For this choice of parameters it also holds that $$d = \Theta \Paren{\frac{n^{\delta/(1+c)}}{\log n}} \quad\quad\text{and}\quad\quad M = \exp\Paren{n^{\delta/(1+c)}} \,.$$
    Let $c_+ = 0, c_- = 1/2, \bm w \in \cU\Paren{\mathbb{S}^{n-1}}$ and $$\beta = \frac 1 {\poly(n)}\,, \quad \gamma = 2 \sqrt{n}\,.$$
    First, consider $$D_+ = \nhclwed{\bm w}{\beta}{\gamma}{c_+}\,, \quad\quad D_- = \nhclwed{\bm w}{\beta}{\gamma}{c_-}\,.$$
    We then set $$D = \frac 1 2 \cdot \Paren{D_+, +1} + \frac 1 2 \cdot \Paren{D_-, -1}\,.$$

    Combining \cref{assump:disting_clwe} and \cref{thm:clwe_to_truncated_clwe} it follows that there is no $O\Paren{\exp\Paren{n^\delta}}$-time distinguisher between $D$ and $N\Paren{0, \tfrac 1 {2\pi} \cdot I_n} \times \mathrm{Be}\Paren{\tfrac 1 2}$ which uses at most $m = O\Paren{\exp\Paren{n^\delta}}$ samples and has non-negligible advantage.
    Let $D'$ be the distribution obtained when applying \cref{lem:hard_ptf_implies_hard_ltf} to $D$.
    Assume towards a contraction that there is a learning algorithm that using time and samples (from $D'$) $$\exp\Paren{O\Paren{\log^{1+c}\Paren{M}}} = \exp \Paren{n^\delta}$$ outputs a binary function $f \colon \supp\Paren{D'} \rightarrow \Set{-1,+1}$ such that $$\err_{D'}\Paren{f} \leq \frac 1 2 - \exp\Paren{-\Omega\Paren{\log^{1+c}\Paren{M}}} = \frac 1 2 - \tau \,.$$

    We claim that we can use this to correctly determine the distribution of the above distinguishing problem in time $O\Paren{\exp\Paren{n^\delta}}$ and with probability at least $2/3$.
    Indeed, suppose we are given $m$ samples from one of the two distributions.
    Note that in case they came from $N\Paren{0, \tfrac 1 {2\pi} \cdot I_n} \times \mathrm{Be}\Paren{\tfrac 1 2}$ the label of the resulting distribution will still be distributed as $\mathrm{Be}\Paren{\tfrac 1 2}$ independently of the example.
    We first transform the samples using the mapping of \cref{lem:hard_ptf_implies_hard_ltf} and then run our learning algorith on the first $m/2$ samples to obtain a hypothesis $f$ with the guarantees above - for simplicity, assume that $m$ is even.
    Next, we compute $$\widehat{\err\Paren{f}} = \frac 2 m \sum_{i=m/2}^m \Ind\Paren{f(x_i) \neq y_i} \,.$$
    If $$\Abs{\widehat{\err\Paren{f}} - \frac 1 2} > \frac \tau 2$$ we output $D$ and else we output $N\Paren{0, \tfrac 1 {2\pi} \cdot I_n} \times \mathrm{Be}\Paren{\tfrac 1 2}$.
    Suppose for now, that the samples come from the distribution $D$.   
    Then by assumption our learning algorithm outputs a hypothesis $h$ such that $$\err_{D'}\Paren{f} \leq \frac 1 2 - \tau\,.$$ 
    Note that $\widehat{\err\Paren{f}}$ is a sum of independent random variables bounded between 0 and 1 and with mean $\err\Paren{f}$.
    Hence, by Hoeffding's Inequality~\cite{hoeffding} it follows that $$\Psymb \Paren{\Abs{\widehat{\err\Paren{f}} - \err\Paren{f}} > \frac \tau 3} \leq 2 \exp \Paren{-\tfrac {2m} 9 \cdot \tau^2} \leq \frac 1 3 \,,$$ where we used that $c_\tau$ is a small enough absolute constant.
    Hence, with probability at least $2/3$ we have that $$\Abs{\widehat{\err\Paren{f}} - \frac 1 2} \geq \Abs{\err\Paren{f} - \frac 1 2} - \Abs{\err\Paren{f} - \widehat{\err\Paren{f}}} \geq \frac{2\tau}{3} > \frac \tau 2\,.$$
    Similary, if the samples come from $N\Paren{0, \tfrac 1 {2\pi} \cdot I_n} \times \mathrm{Be}\Paren{\tfrac 1 2}$ it follows that $\Psymb_{(\bm x', y) \sim D'} \Paren{f\Paren{\bm x'} \neq y} = 1/2$ and hence $$\Psymb \Paren{\Abs{\widehat{\err\Paren{f}} - 1/2} > \frac \tau 3} \leq \frac 1 3 \,.$$
    Together this yields that $$\Abs{\widehat{\err\Paren{f}} - \frac 1 2} \leq \frac \tau 3 < \frac \tau 2$$ with probability at least $2/3$.

    Next, we will show the second part of the theorem.
    To this end, note that from \cref{eq:def_n} it follows that $$n = \log^{\tfrac{1+c}{\delta}}\Paren{M}$$ and hence $$\frac d \gamma = \Omega\Paren{\frac{\log^{\Paren{1-\tfrac{1+c}{2\delta}}} M}{\log \log M}} = \Omega\Paren{\frac{\log^{\tfrac 1 2 \cdot \Paren{1-c''}} M}{\log \log M}} = \Omega(1) \,,$$ where we used that $\tfrac{1+c}{1+c''} < \delta$ implies that $$1-\tfrac{1+c}{2\delta} > \tfrac 1 2 - \tfrac 1 2 \cdot c'' > 0\,.$$
    Hence, from \cref{lem:nhclwe_is_agnostic_ptf} it follows that there exists a degree-$4d$ PTF $h^*$ satisfying $$\err_D\Paren{h^*} \leq \exp\Paren{- \Omega \Paren{\frac{d^2}{\gamma^2}}} = \exp\Paren{-\frac{\log^{1-c''} M}{\log \log M}} = \exp\Paren{-\Omega\Paren{\log^{1-c'} M}}\,,$$ for $c'$ slightly larger than $c''$.
    Further, it holds that $$\forall \bm x \in \R^n \colon \quad \Psymb_{\paren{\bm x,y} \sim D} \Paren{h^*(\bm x) \neq y \given \bm x} \in \Set{0,1} \,.$$
    By \cref{lem:hard_ptf_implies_hard_ltf} it follows that for the same distribution $D'$ there exists a linear threshold function $f^* \colon \R^M \rightarrow \Set{-1,+1}$ which has the same misclassification error and conditional error probabilites (with respect to $D'$) which finishes the proof.
\end{proof}

It remains to prove \cref{lem:nhclwe_is_agnostic_ptf}
\begin{proof}[Proof of \cref{lem:nhclwe_is_agnostic_ptf}]
    Let $d, n \in \N, \beta, \gamma \in \R_{> 0}$ such that $$\beta^2 \leq \gamma^2\,, \quad\quad  \frac d \gamma = \Omega(1) \,.$$
    Further, let $c_+ = 0, c_- = 1/2$ and $\bm w \in \mathbb{S}^{n-1}$.
    Recall that $$D_+ = \nhclwed{\bm w}{\beta}{\gamma}{c_+}\,, \quad\quad D_- = \nhclwed{\bm w}{\beta}{\gamma}{c_-}\,.$$
    We will first show that for our choice of parameters the supports of $D_+$ and $D_-$ are disjoint. 
    To this end, recall that for $c \in[0,1)$ the distribution $\nhclwed{\bm w}{\beta}{\gamma}{c}$ has density proportional to
    \begin{align*}
        \sum_{k \in \Z} \rho_{\sqrt{\beta^2 + \gamma^2}}(k\,;c) \cdot \rho\Paren{\pi_{{\bm w}^\perp}(\bm y)} \cdot \rho_{\beta / \sqrt{\beta^2 + \gamma^2}}^\alpha\Paren{\iprod{\bm w, \bm y} \,; \frac{\gamma}{\beta^2 + \gamma^2} \Paren{k-c} } \,,
    \end{align*}
    where $\alpha = \tfrac 1 {10} \cdot \tfrac{\gamma}{\gamma^2+\beta^2}$ and $\pi_{{\bm w}^\perp} \Paren{\bm y}$ denotes the projection of $\bm y$ onto the orthogonal complement of $\bm w$.
    For $k \in \Z$ let $$\mu_k^+ = \frac{\gamma}{\beta^2 + \gamma^2} \Paren{k-c_+} = \frac{\gamma}{\beta^2 + \gamma^2}k \quad\quad\text{and}\quad\quad \mu_k^- = \frac{\gamma}{\beta^2 + \gamma^2} \Paren{k-c_-} = \frac{\gamma}{\beta^2 + \gamma^2} \Paren{k-\frac 1 2}\,.$$
    Consider the intervals
    \begin{align*}
        J^+_k &= \Brac{\mu_k^+ - \alpha, \mu_k^+ + \alpha} \,, \\
        J^-_k &= \Brac{\mu_k^- - \alpha, \mu_k^- + \alpha} \,.
    \end{align*}
    Then it follows that
    \begin{align*}
        \supp\Paren{D_+} &= \bigcup_{k \in \Z} \; \Set{\bm x \in \R^n \suchthat \iprod{\bm w, \bm x} \in J^+_k} \,, \\
        \supp\Paren{D_-} &= \bigcup_{k \in \Z} \; \Set{\bm x \in \R^n \suchthat \iprod{\bm w, \bm x} \in J^-_k} \,.
    \end{align*}
    Since the intervals $J_k^+, J_k^-$ are symmetric around $\mu_k^+$ and $\mu_k^-$ respectively and $$\min \Set{\Abs{\mu_k^+ - \mu_k^-}, \Abs{\mu_k^+ - \mu_{k+1}^-}} = \frac 1 2 \cdot \frac{\gamma}{\beta^2 + \gamma^2} \,,$$
    it follows that the supports of $D_+$ and $D_-$ are disjoint if and only if $$\frac 1 2 \cdot \frac{\gamma}{\beta^2 + \gamma^2} > 2\alpha = \frac 1 5 \cdot \frac \gamma {\beta^2 + \gamma^2} \,,$$
    which always is the case.
    Hence, the supports of $D_+$ and $D_-$ are indeed disjoint.

    Consider next the $2d$ intervals $J^-_{-d+1}, \ldots, J^-_{d}$ and the minimum-degree polynomial $p_{\bm w} \colon \R \rightarrow \R$  that is zero on exactly the points halfway between one of these intervals and the closest $J^+_k$ intervals.
    Further, choose this in such a way that it is non-positive on $J^-_{-d+1}, \ldots, J^-_{d}$.
    Note by construction it has degree $4d$.
    Further, consider the degree-$4d$ PTF
    \[
    \begin{aligned}
        p \colon \R^n &\rightarrow \R\,, \\
        \bm x &\mapsto \sign\Paren{p_{\bm w}(\iprod{\bm w, \bm x})} \,.
    \end{aligned}
    \]
    Let $$S^- = \bigcup_{k=-d+1}^d \Set{\bm x \in \R^n \suchthat \iprod{\bm w, \bm x} \in J^-_k}\,.$$
    Note that for all $\bm x$ such that $D_+(\bm x) \neq 0$ it holds that $$\Psymb_{\paren{\bm x,y} \sim D} \Paren{p(\bm x) \neq y \given \bm x} = 0 $$ since for such $\bm x$ the label $y$ is always equal to $+1$ and so is the value of $p$.
    For the same reason the same holds for all $\bm x \in S^-$.
    Hence, we obtain that $$\Psymb_{\paren{\bm x,y} \sim D} \Paren{p(\bm x) \neq y} = \Psymb_{\paren{\bm x,y} \sim D} \Paren{\bm x \in \supp\Paren{D_-} \setminus S^-} \,.$$

    Let $Z =\tfrac{\beta}{\sqrt{\beta^2 + \gamma^2}} \cdot \rho_{\sqrt{\beta^2 + \gamma^2}}\Paren{ \Z\,; c_-}$ then by definition of $D_-$ and using that for $s > 0$ $$\int_{\abs{z - c} \leq \alpha} \rho_s^\alpha(z \, ; c) \,dz = \int_{\R} \rho_s(z\,;c) \,dz = s\,.$$ it follows that
    \begin{align*}
        \Psymb_{\paren{\bm x,y} \sim D} \Paren{\bm x \in \supp\Paren{D_-} \setminus S^-} &= \frac 1 Z \sum_{\substack{k \leq -d \,, \\ k \geq d+1}} \rho_{\sqrt{\beta^2 + \gamma^2}}(k \,; c_-) \cdot \int_{J^-_k}\rho_{\beta / \sqrt{\beta^2 + \gamma^2}}^\alpha\Paren{z \,; \frac{\gamma}{\beta^2 + \gamma^2} \Paren{k-c_-} } \,dz\\
        &\leq \frac 1 Z \sum_{\abs{k} \geq d} \rho_{\sqrt{\beta^2 + \gamma^2}}(k\,; c_-) \cdot \int_{J^-_k} \rho_{\beta / \sqrt{\beta^2 + \gamma^2}}^\alpha\Paren{z \,; \frac{\gamma}{\beta^2 + \gamma^2} \Paren{k-c_-} } \,dz \\
        &= \frac 1 Z \sum_{\abs{k} \geq d} \rho_{\sqrt{\beta^2 + \gamma^2}}(k\,; c_-) \cdot \int_{\R} \rho_{\beta / \sqrt{\beta^2 + \gamma^2}}\Paren{z \,; \frac{\gamma}{\beta^2 + \gamma^2} \Paren{k-c_-} } \,dz  \\
        &= \frac{\beta}{\sqrt{\beta^2+\gamma^2} \cdot Z} \sum_{\abs{k} \geq d} \rho_{\sqrt{\beta^2 + \gamma^2}}(k\,; c_-) \\
        &= \frac 1 {\rho_{\sqrt{\beta^2 + \gamma^2}}\Paren{\Z\,; c_-}} \cdot \sum_{\abs{k} \geq d} \rho_{\sqrt{\beta^2 + \gamma^2}}(k\,; c_-) \,.
    \end{align*}
    It follows that $$\Psymb_{\paren{\bm x,y} \sim D} \Paren{\bm x \in \supp\Paren{D_-} \setminus S^-} \leq \Psymb\Paren{\Abs{U} \geq d}\,,$$ where $U \sim D_{\Z-c_-, \sqrt{\beta^2 + \gamma^2}}$.
    By standard tailbounds for the discrete Gaussian distribution~\cite[Lemma 2.8]{disc_gaussian_tailbound} we conclude that $$\Psymb\Paren{\Abs{U} \geq d} \leq \Theta \Paren{1} \cdot \exp\Paren{-\pi \cdot \frac{d^2}{\beta^2+\gamma^2}}= \exp\Paren{- \Omega \Paren{\frac{d^2}{\gamma^2}}} \,,$$ where in the last equality we used that $\beta^2 \leq \gamma^2$ and $d/\gamma = \Omega(1)$.

    Moreover, all points for which $\Psymb_{\paren{\bm x,y} \sim D} \Paren{h^*(\bm x) \neq y \given \bm x} \neq 0$ have their projection onto $\bm w$ in $\supp\Paren{D_-} \setminus S^-$.
    However, since for such $\bm x$ the distribution $D$ always outputs a $-1$ label, whereas $p\Paren{\bm x} = +1$, it follows that for such $\bm x$ $$\Psymb_{\paren{\bm x,y} \sim D} \Paren{h^*(\bm x) \neq y \given \bm x} = 1\,.$$
\end{proof}

\section{Hardness of Distribution-Specific Learning}
\label{sec:distribution-specific}

In this section, we show hardness results for agnostic learning even when the $\bm x$ marginal distribution is Gaussian based on \cref{assump:disting_clwe}.
For consistency with the rest of the paper, we show the result where the marginal distribution is equal to $N(0, \tfrac 1 {2\pi} \cdot I_M)$ instead of standard Gaussian.
Recall that for a distribution $D$ over $\R^M \times \Set{-1,+1}$, we denote by $D_{\bm x}$ its marginal distribution over $\R^M$.
More specifically, we will show:
\begin{theorem}
    \label{thm:distribution_specific}
    Let $M \in \N$ and $\e > 0$ be small enough.
    There exists a distribution $D$ over $\R^M \times \Set{-1, +1}$ such that $D_{\bm x} = N(0, \tfrac 1 {2\pi} \cdot I_M)$ and under \cref{assump:disting_clwe} for all $\delta < 1$, there is no algorithm using fewer than $$M^{\Omega\Paren{\tfrac 1 {\log(1/\epsilon)} \cdot \Paren{\tfrac 1 {\epsilon^2}}^\delta}}$$ time and samples that outputs any binary hypothesis $f$ such that \[\err_D\Paren{f} \leq \OPT_{\LTF} + \e \,.\]
    Further, under the same assumption, there is no algorithm using fewer than $$M^{\Omega\Paren{\tfrac 1 {\log(\ell/\epsilon)} \cdot \Paren{\tfrac {\ell^2} {\epsilon^2}}^\delta}}$$ time and samples that outputs any binary hypothesis $f$ such that \[\err_D\Paren{f} \leq \OPT_{\PTF_\ell} + \e \,.\]
\end{theorem}

The hard distribution $D$ is defined as follows:
Let $\gamma = 2\sqrt{M}, \beta = \tfrac 1 \poly(M)$ and $\bm w$ be uniform over $\bbS^{M-1}$.
\begin{itemize}
    \item Draw a sample $(\bm x, z) \sim \clwed{\bm w}{\beta}{\gamma}$.
    \item If $z \in [0,1/2)$ output $(\bm x, +1)$, else output $(\bm x, -1)$.
\end{itemize}
\cref{thm:distribution_specific} will follow directly by the following two lemmas.
\begin{lemma}
    \label{lem:distribution_specific_hard_to_distinguish}
    Let $D$ be as defined above and $\delta < 1$.
    Then $D_{\bm x} = N(0,\tfrac 1 {2\pi} \cdot I_M)$ and under \cref{assump:disting_clwe} there is no algorithm that uses fewer than $2^{M^\delta}$ time and samples and can distinguish $D$ from $\N(0,\tfrac 1 {2\pi} \cdot I_M) \times \mathrm{Be}\Paren{\tfrac 1 2}$ with non-negligible advantage in $M$.
\end{lemma}
\begin{lemma}
    \label{lem:distribution_specific_learner_can_distinguish}
    Let $D$ again be as above and $\e > 0$ be small enough.
    Suppose there is an algorithm using fewer than $$M^{\Omega\Paren{\tfrac 1 {\log(1/\epsilon)} \cdot \Paren{\tfrac 1 {\epsilon^2}}^\delta}}$$ time and samples that outputs a binary hypothesis $f$ such that $\err_D\Paren{f} \leq \OPT_\LTF + \e$.
    Then there is an algorithm that uses the same amount of time and samples that can distinguish $D$ from $N(0,\tfrac 1 {2\pi} \cdot I_M) \times \mathrm{Be}\Paren{\tfrac 1 2}$ with non-negligible advantage.
    Similarly, if there is an algorithm using fewer than $$M^{\Omega\Paren{\tfrac 1 {\log(\ell/\epsilon)} \cdot \Paren{\tfrac \ell {\epsilon^2}}^\delta}}$$ time and samples that outputs a binary hypothesis $f$ such that $\err_D\Paren{f} \leq \OPT_{\PTF_\ell} + \e$.
    Then there is an algorithm that uses the same amount of time and samples that can distinguish $D$ from $N(0,\tfrac 1 {2\pi} \cdot I_M) \times \mathrm{Be}\Paren{\tfrac 1 2}$ with non-negligible advantage.
\end{lemma}

We start by proving \cref{lem:distribution_specific_hard_to_distinguish}.
\begin{proof}[Proof of \cref{lem:distribution_specific_hard_to_distinguish}]
    We first show that $D_{\bm x} = N(0, \tfrac 1 {2\pi} \cdot I_M)$.
    Let $p$ be the density of $\clwed{\bm w}{\beta}{\gamma}$ and $p_D$ the density of $D$.
    For $(\bm x, y) \in \R^M \times \Set{-1, +1}$ it holds that \[ p_D\Paren{\bm x \suchthat y} = \begin{cases}
        p \Paren{\bm x \suchthat z \in [0,1/2)} \,, &\quad \text{if $y = +1$,} \\
        p \Paren{\bm x \suchthat z \in [1/2,1)} \,, &\quad \text{if $y = -1$.}
    \end{cases}\]
    Let $\bm x \in \R^M$, we compute the density $p_{\bm x}(\bm x)$ of $D_{\bm x}$ at point $\bm x$.
    \begin{align*}
        p_{\bm x}(\bm x) &= \frac{1}{2} \cdot p \Paren{\bm x \suchthat z \in [0,1/2)} + \frac{1}{2} \cdot p \Paren{\bm x \suchthat z \in [1/2,1)} = \int_0^1 p(\bm x, c) \, dc \\
        &= \rho(\bm x) = N(0, \tfrac 1 {2\pi} \cdot I_M) (\bm x)\,.
    \end{align*}

    Further, let $m = 2^{M^\delta}$.
    Given a $T$-time distinguisher $\cA$ between
    \[
        D^m  \quad\text{and}\quad N(0, \tfrac 1 {2\pi} \cdot I_M)^m \times \mathrm{Be}\Paren{\tfrac 1 2}^m
    \]
    we construct a $O(T)$-time distinguisher between
    \[
        \clwem{\bm w}{\beta}{\gamma} \quad\text{and}\quad N(0, \tfrac 1 {2\pi} \cdot I_M)^m \times U\Paren{[0,1)}^m \,.
    \]

    Given samples $(\bm x, z)$ from either $\clwed{\bm w}{\beta}{\gamma}$ or $N(0, \tfrac 1 {2\pi} \cdot I_M) \times U\Paren{[0,1)}$, we construct new samples $(\bm x', y')$ as follows.
    \begin{itemize}
        \item If $z \in [0,1/2)$ output $(\bm x, +1)$,
        \item else output $(\bm x, -1)$.
    \end{itemize}
    In case $(\bm x, z)$ came from $\clwed{\bm w}{\beta}{\gamma}$, $(\bm x', y')$ will be distributed according to $D$ by definition.
    In case $(\bm x, z)$ came from $N(0, \tfrac 1 {2\pi} \cdot I_M) \times U\Paren{[0,1)}$, $\bm x'$ and $y'$ will be independent and with marginals $N(0, \tfrac 1 {2\pi} \cdot I_M)$ and $\mathrm{Be}\Paren{\tfrac 1 2}$, respectively, as desired.
    Hence, we can directly use our distinguisher $\cA$ to distinguish the two cases.
\end{proof}

Next, we will proof \cref{lem:distribution_specific_learner_can_distinguish}
\begin{proof}[Proof of \cref{lem:distribution_specific_learner_can_distinguish}]
    We start by proving the result about LTFs, the result about PTFs will follow in the same way.
    Let $\delta > 0$ and $\tau = \tfrac{1}{\poly(M)}$.
    Suppose $\OPT_\LTF$ and $\e$ are such that $$\err_D\Paren{f} \leq \OPT_\LTF + \e \leq \tfrac 1 2 - \tau \,.$$
    We proceed similarly to the proof of \cref{thm:agnostic_ltf}.
    Given $m = 2^{M^\delta}$\footnote{For simplicity assume that $m$ is even.} samples $(\bm x_1, y_1), \ldots (\bm x_m, y_m)$ from either $D$ or $N(0, \tfrac 1 {2\pi} \cdot I_M)\times \mathrm{Be}\Paren{\tfrac 1 2}$ we run our algorithm on the first $m/2$ samples to obtain a binary hypothesis $f$.
    Let $$\widehat{\err\Paren{f}} = \frac{2}{m} \sum_{i=m/2 + 1}^m \Ind(f(\bm x_i) \neq y_i)\,.$$
    If $\abs{\widehat{\err\Paren{f}} - \tfrac 1 2 } > \frac \tau 2$, output $D$, else output $N(0, \tfrac 1 {2\pi} \cdot I_M)\times \mathrm{Be}\Paren{\tfrac 1 2}$.
    By an application of Hoeffding's Inequality, it follows as in the proof of \cref{thm:agnostic_ltf}, that this test successfully distinguishes between the two distributions with probability at least $2/3$.

    Assume, that for an absolute constant $c > 0$, it holds that
    \[
        \OPT_\LTF \leq \frac 1 2 - \frac c \gamma = \frac 1 2 - \frac c {2 \sqrt{M}}\,.
    \]
    We will verify this shortly.
    This implies, that we can choose $\e = \Omega\Paren{1 / \sqrt{M}}$ and it still holds that $\OPT_\LTF + \e \leq \tfrac 1 2 - \tau$.
    Since
    \[
        2^{M^\delta} = M^{\Omega\Paren{\frac 1 {\log(1/\e)} \cdot \Paren{\frac 1 {\e^2}}^\delta}} \,,
    \]
    the result will follow.

    We next turn to bounding $\OPT_\LTF$.
    First, note that the density of $D$ is equal to
    \begin{align*}
        p_D(\bm x, y) &=  \frac 1 \beta \rho(\bm x) \cdot \begin{cases}
            \sum_{k \in \Z} \int_0^{1/2} \rho_{\beta}\Paren{ c +k - \gamma \iprod{\bm w, \bm x}} \, dc\,, &\quad \text{if $y = +1$,} \\
            \sum_{k \in \Z} \int_{1/2}^1 \rho_{\beta}\Paren{ c +k - \gamma \iprod{\bm w, \bm x}} \, dc\,, &\quad \text{if $y = -1$.} 
        \end{cases}
    \end{align*}
    To simplify the analysis we will work with the distribution $D'$ whose density is equal to
    \begin{align*}
        p_{D'}(\bm x, y) &= \rho(\bm x) \cdot \begin{cases} \sum_{k \in \Z} \Ind\Paren{\gamma \iprod{\bm w, \bm y} \in [k, k+1/2)}\,, &\quad \text{if $y = +1$,} \\
            \sum_{k \in \Z} \Ind\Paren{\gamma \iprod{\bm w, \bm y} \in [k+1/2, k+1)} \,, &\quad \text{if $y = -1$.} 
        \end{cases}
    \end{align*}
    By \cref{lem:tvd_distribution_specific} it holds that $\TVD{D}{D'} \leq \tfrac{1}{\poly(M)}$ and hence if $\tilde{f}$ is any linear threshold function it holds that \[\err_D\Paren{\tilde{f}} = \Psymb_{(\bm x, y) \sim D}\Paren{\tilde{f}(\bm x) \neq y} \leq \Psymb_{(\bm x, y) \sim D'}\Paren{\tilde{f}(\bm x) \neq y} + \frac{1}{\poly(M)} = \err_{D'} \Paren{\tilde{f}} + \frac{1}{\poly(M)}\,.\]
    Finally, in \cref{lem:distribution_specific_optimal_ltf} we show that there exists a halfspace $f^*$ and an absolute constant $c' > 0$ such that $\err_{D'}\Paren{f^*} \leq \tfrac 1 2 - \frac {c'} \gamma$ implying that there is a second absolute constant $c> 0$ such that 
    \[
        \OPT_\LTF \leq \err_D\Paren{f^*} \leq \frac 1 2 - \frac {c} \gamma \,.
    \]

    Next, we turn to the result about PTFs.
    Analogously as above, it follows from \cref{lem:distribution_specific_optimal_ltf} that there exists a degree-$\ell$ PTF, such that $\err_{D'}\Paren{f^*} \leq \tfrac 1 2 - \frac {c' \ell} \gamma$.
    Hence, in this case, there is an absolute constant $c > 0$, such that $\OPT_{\PTF_\ell} \leq \frac 1 2 - \frac {c \ell} \gamma$.
    Hence, we can choose $\e = \Omega\Paren{\ell / \sqrt{M}}$, implying that
    \[
        2^{M^\delta} = M^{\Omega\Paren{\frac 1 {\log(\ell/\e)} \cdot \Paren{\frac \ell \e}^{2\delta}}}
    \]
\end{proof}
\section{Reductions}

The goal of this section is to show that the mixture distributions which we used to prove hardness of learning in the agnostic model  are hard to learn under \cref{assump:disting_clwe}.
In particular, our goal will be to prove the following theorem
\begin{theorem}
    \label{thm:clwe_to_truncated_clwe}
    Let $n, m \in \N$ with $2^n > m > n$, and let $\gamma, \beta, \e \in \R_{> 0}, c \in [0,1)$ such that
    \begin{align*}
        0 &\leq \beta \leq \gamma \,,\\
        \beta &= \tfrac 1 {\poly(n)} \,.
    \end{align*}
    Assume that there is no $(T + \poly(n,m))$-time distinguisher between
    \[
    \begin{aligned}
        \clwem{m}{\gamma}{\beta} \quad&\text{and}\quad N\Paren{0, \tfrac 1 {2\pi} \cdot I_n}^m \times U\Paren{[0,1)}^m
    \end{aligned}
    \]
    with advantage $\e$.
    Let $m' = \tfrac m {\poly(n)}$.
    Then, there is no $T$-time distingiusher between
    \[
    \begin{aligned}
        D_c \coloneqq \nhclwem{m'}{\gamma}{2\beta}{c} \quad&\text{and}\quad N\Paren{0, \tfrac 1 {2\pi} \cdot I_n}^m   
    \end{aligned}
    \]
    with advantage $\e-\negl(n)$.
    Moreover, let $c_+,c_- \in [0,1)$.
    Then there is no $T$-time distingiusher between
    \[
    \begin{aligned}
        \frac 1 2 \cdot \Paren{D_{c_+}, +1} + \frac 1 2 \cdot \Paren{D_{c_-}, -1} \quad&\text{and}\quad N\Paren{0, \tfrac 1 {2\pi} \cdot I_n} \times \mathrm{Be}\Paren{\frac 1 2}        
    \end{aligned}
    \]
    with advantage $\e-\negl(n)$ that uses at most $m'$ samples.
\end{theorem}

One key ingredient is the following straightforward adaptation of~\cite[Lemma 4.1]{CLWE}.
We will include its proof for completeness at the end of this section.

\begin{lemma}[straightforward extension of Lemma 4.1 in~\cite{CLWE}]
    \label{lem:clwe_to_hclwe}
    For every $\bm w \in \R^n$ there is a $\poly(n, 1/\delta)$-time probabilistic algorithm that takes as input parameters $\delta \in (0,1), c \in [0,1)$, and samples from $\clwed{\bm w}{\gamma}{\beta}$ and outputs samples from $\hclwed{\bm w}{\sqrt{\beta^2 + \delta^2}}{\gamma}{c}$.
    More specifically, given $\poly(n, 1/\delta)$ CLWE samples the algorithm runs in time $\poly(n, 1/\delta)$ and with probability at least $1-\exp(-\poly(n, 1/\delta))$ outputs at least one HCLWE sample.
    Further, if given samples from $N\Paren{0, \tfrac 1 {2\pi} \cdot I_n} \times \cU\Paren{[0,1)}$ the procedure will output samples from $N\Paren{0, \tfrac 1 {2\pi} \cdot I_n}$.
\end{lemma}

With this in had, we can prove \cref{thm:clwe_to_truncated_clwe}
\begin{proof}[Proof of \cref{thm:clwe_to_truncated_clwe}]
    Let $\gamma, \beta, \e \in \R_{> 0}, c \in [0,1)$.
    Assume that there is no $(T + \poly(n,m))$-time distinguisher between $$\clwem{m}{\gamma}{\beta} \quad\text{and}\quad N\Paren{0, \tfrac 1 {2\pi} \cdot I_n}^m \times U\Paren{[0,1)}^m$$ with advantage $\e$.
    Let $m' = \tfrac m {\poly(n)}$.
    We claim that this implies that there is no $T$-time distinguisher between $$\hclwem{m'}{\gamma}{2\beta}{c} \quad\text{and}\quad N\Paren{0, \tfrac 1 {2\pi} \cdot I_n}^{m'}$$ with advantage $\e - \negl(n)$.
    Note that this implies the conclusion of the theorem since by the second part of \cref{lem:tvd_hclwe_and_nhclwe} the total variation distance between $\hclwed{\bm w}{\beta}{\gamma}{c}$ and $\nhclwed{\bm w}{\gamma}{\beta}{c}$ is at most $$4 \cdot \exp \Paren{- \frac{1}{100\beta^2}}$$ for every $\bm w$ that is unit.
    Hence, since $m < 2^n$ the total variation distance between the respective $m'$-fold product distributions is at most $$4m \cdot \exp \Paren{- \frac{1}{100\beta^2}} = \negl(n)\,.$$
    Note that we can apply this since in our case $0 \leq \beta \leq \gamma$.
    See \cref{lem:small_tvd_advantage} for a formal proof of the fact that a small change in total variation distance results in only a small change in the distinguishing advantage.

    To show the claim, we will use \cref{lem:clwe_to_hclwe}.
    Concretely, assume that there is a $T$-time distinguisher between $$\hclwem{m'}{\gamma}{2\beta}{c} \quad\text{and}\quad N\Paren{0, \tfrac 1 {2\pi} \cdot I_n}^{m'}\,.$$
    We will use this to build a $(T + \poly(n,m))$-time distinguisher between $$\clwem{m}{\gamma}{\beta} \quad\text{and}\quad N\Paren{0, \tfrac 1 {2\pi} \cdot I_n}^m \times U\Paren{[0,1)}^m$$ as follows:
    Let $\bm w$ denote the secret vector of the $\clwe$ distribution.
    Given $m$ samples from either $\clwed{\bm w}{\gamma}{\beta}$ or $N\Paren{0, \tfrac 1 {2\pi} \cdot I_n} \times U([0,1))$ we invoke the algorithm of \cref{lem:clwe_to_hclwe} with $$\delta = \sqrt{3}\beta = \Omega(1/\poly(n))\,.$$
    In case the samples came from $\clwed{\bm w}{\gamma}{\beta}$ with probability at least $$1 - \exp(-\poly(n, 1/\delta)) = 1 - \negl(n)$$ we obtain in time $\poly(n)$ at least $m' = \tfrac m {\poly(n)}$ samples from $\hclwed{\bm w}{2\beta}{\gamma}{c}$.
    In case the samples came from $N\Paren{0, \tfrac 1 {2\pi} \cdot I_n}^n \times U([0,1))$ with at least the same probability we obtain in time $\poly(n)$ at least $m' = \tfrac m {\poly(n)}$ samples from $D_1^n$.
    Hence, if we had a $T$-time distinguisher between $$\hclwem{m'}{\gamma}{2\beta}{c} \quad\text{and}\quad N\Paren{0, \tfrac 1 {2\pi} \cdot I_n}^{m'}$$ with advantage $\e - \negl(n)$, this would directly yield a $(T + \poly(n,m))$-time distinguisher between $$\clwem{m}{\gamma}{\beta} \quad\text{and}\quad N\Paren{0, \tfrac 1 {2\pi} \cdot I_n}^m \times U\Paren{[0,1)}^m$$ with advantage $\e$.
    The shift of $\negl(n)$ in the advatage is due to the fact that the sample conversion algortihm can fail with probability $\negl(n)$.

    For the second part of the theorem, we first note, that we can again replace the truncated mixture distributions by the non-truncated ones by invoking \cref{lem:small_tvd_advantage}.
    By construction, the respective mixture distributions have total variation distance at most $$p \cdot \negl(n) + (1-p) \cdot \negl(n) = \negl(n)\,.$$
    The result follows since we can generate samples from this mixture from samples from the CLWE distribution as follows:
    With probabily $p$ we invoke the procedure of \cref{lem:clwe_to_hclwe} with $c = c_+$ and with probability $1-p$ with $c = c_-$.
\end{proof}

Finally, we give the proof of \cref{lem:clwe_to_hclwe}

\begin{proof}
    Let $\delta \in (0,1), c \in [0,1)$ and $\beta, \gamma > 0$.
    Without loss of generality assume that $\bm w = \bm e_1$.
    Given samples from $\clwed{\bm w}{\beta}{\gamma}$ the idea is to perform rejection sampling to obtain samples from $\hclwed{\bm w}{\sqrt{\beta^2 + \delta^2}}{\gamma}{c}$.
    Concretely, let $g \colon [0,1) \rightarrow [0,1]$ be given by $g(z) = g_0(z)/M$, where $$g_0(z) = \sum_{k \in \Z} \rho_\delta(z+k+c) \,,\quad\quad M = \sup_{z\in [0,1)} g_0(z)\,.$$
    For a CLWE sample $(\bm y, z)$, output $\bm y$ with probability $g(z)$.\footnote{Note that by \cite[Section 5.2]{classical_lwe} the function $g(z)$ is efficiently computable.}
    Recall that the density of $\clwed{\bm w}{\beta}{\gamma}$ is given by
    \begin{align*}
        p(\bm y, z) = \frac{1}{\beta} \cdot  \rho \Paren{\bm y} \cdot \sum_{k \in \Z} \rho_\beta \Paren{z + k - \gamma \bm y_1} \,.
    \end{align*}
    Using \cref{fact:rho_fact} (in the third equality) and that for all $c \in \R$ it holds that $\int_{\R} \rho_s(x-c) \, dx = s$ we obtain that the density $p'$ of the distribution given by the rejection sampling, i.e., of outputting $\bm y$ and accept, is given by
    \begin{align*}
        &p'(\bm y) = \int_{[0,1)} p(\bm y, z) g(z) \,dz = \frac{\rho(\bm y)}{\beta \cdot M} \cdot \int_{[0,1)} \sum_{k_1, k_2 \in \Z} \rho_\beta\Paren{z + k_1 - \gamma \bm y_1} \cdot \rho_\delta\Paren{z+k_2+c} \, dz \\
        &= \frac{\rho(\bm y)}{\beta \cdot M} \cdot \int_{[0,1)} \sum_{k_1, k_2 \in \Z} \rho_{\sqrt{\beta^2 + \delta^2}} \Paren{\gamma \bm y_1 - k_1 + k_2 + c} \rho_{\beta \delta / \sqrt{\beta^2 + \delta^2}} \Paren{z + \tfrac{\beta^2}{\beta^2 + \gamma^2}(k_2 + c) + \tfrac{\gamma^2}{\beta^2 + \gamma^2}(k_1 - \gamma \bm y_1)} \, dz \\
        &= \frac{\rho(\bm y)}{\beta \cdot M} \cdot \int_{[0,1)} \sum_{k, k_2 \in \Z} \rho_{\sqrt{\beta^2 + \delta^2}} \Paren{\gamma \bm y_1 - k + c} \cdot \rho_{\beta \delta / \sqrt{\beta^2 + \delta^2}} \Paren{z + k_2 + \tfrac{\beta^2}{\beta^2 + \delta^2} c + \tfrac{\delta^2}{\beta^2 + \delta^2}(k - \gamma \bm y_1)} \, dz \\
        &= \frac{\rho(\bm y)}{\beta \cdot M} \cdot \sum_{k\in \Z} \rho_{\sqrt{\beta^2 + \delta^2}} \Paren{\gamma \bm y_1 - k + c} \cdot \int_{\R} \rho_{\beta \delta / \sqrt{\beta^2 + \delta^2}} \Paren{x +  \tfrac{\beta^2}{\beta^2 + \delta^2} c + \tfrac{\delta^2}{\beta^2 + \delta^2}(k - \gamma \bm y_1)} \, dx \\
        &= \frac{\delta \cdot\rho(\bm y)}{\sqrt{\beta^2 + \delta^2} \cdot M} \sum_{k\in \Z} \rho_{\sqrt{\beta^2 + \delta^2}} \Paren{\gamma \bm y_1\,; k-c}  \,.
    \end{align*}
    Hence, the distribution is indeed equal to $\hclwed{\bm w}{\sqrt{\beta^2 + \delta^2}}{\gamma}{c}$.
    It also follows, that the probability that we accept a given CLWE sample is equal to 
    \begin{align*}
        \int_{\R^n} p'(\bm y) \,d \bm y &=  \frac{\delta}{\sqrt{\beta^2 + \delta^2} \cdot M} \cdot \frac{\sqrt{\beta^2 + \delta^2}}{\sqrt{\beta^2 + \delta^2 + \gamma^2}} \cdot \rho\Paren{\tfrac{1}{\sqrt{\beta^2 + \delta^2 + \gamma^2}} \Z} \\
        &= \frac{\delta}{\sqrt{\beta^2 + \delta^2 + \gamma^2} \cdot M} \cdot \rho\Paren{\tfrac{1}{\sqrt{\beta^2 + \delta^2 + \gamma^2}} \Z} \\
        &= \frac{\delta}{\sqrt{\beta^2 + \delta^2 + \gamma^2} \cdot M} \cdot \rho_{\sqrt{\beta^2 + \delta^2 + \gamma^2}}\Paren{\Z} 
    \end{align*}

    Note that using \cref{fact:poisson_summation} it follows that $$\rho_{\sqrt{\beta^2 + \delta^2 + \gamma^2}}\Paren{\Z} = \sqrt{\beta^2 + \delta^2 + \gamma^2} \cdot \rho_{1/\sqrt{\beta^2 + \delta^2 + \gamma^2}}\Paren{\Z} \geq \sqrt{\beta^2 + \delta^2 + \gamma^2}\,.$$
    Hence, the probability that we accept is at least $\delta/M$.
    Further, for each $z \in [0,1)$ we have that
    \begin{align*}
        g_0(z) &= \sum_{k \in \Z} \rho_\delta(z+k+c) \\
        &\leq 2 \cdot \sum_{k=0}^\infty \rho_\delta(k) \\
        &< 2 \cdot \sum_{k=0}^\infty \exp \Paren{-\pi k} < 4 \,.
    \end{align*}
    Hence, $M \leq 4$ and it follows that we accept with probability at least $\delta/4$.
    Thus, after $\poly(n, 1/\delta)$ we output at least one HCLWE sample with probability at least $1-\exp(-\poly(n, 1/\delta))$.

    Lastly, when given samples from $N\Paren{0, \tfrac 1 {2\pi} \cdot I_n} \times \cU\Paren{[0,1)}$ the procedure will output samples from $N\Paren{0, \tfrac 1 {2\pi} \cdot I_n}$ since in this case $\bm y$ and $z$ are independent.
\end{proof}

\newpage

\phantomsection
\addcontentsline{toc}{section}{References}
\bibliographystyle{amsalpha}
\bibliography{bib/custom,bib/dblp,bib/mathreview,bib/scholar}

\appendix

\section{Missing Lemmas}
\label{sec:missing_lemmas}

\subsection{TVD Closeness of Supporting Distributions}

\begin{lemma}
    \label{lem:tvd_hclwe_and_nhclwe}
    Let $\bm w \in \R^d$ and $0 \leq \beta \leq \gamma, c \in [0,1)$.
    Then $$\TVD{\hclwed{\bm w}{\beta}{\gamma}{c}}{\nhclwed{\bm w}{\beta}{\gamma}{c}} \leq 4 \cdot \exp \Paren{-\frac 1 {100 \beta^2}}\,.$$
\end{lemma}

\begin{proof}
    Let $P_H$ denote the density of $\hclwed{\bm w}{\beta}{\gamma}{c}$ and $P_N$ the density of $\nhclwed{\bm w}{\beta}{\gamma}{c}$.
    Abusing notation slightly, we also use $P_H$ and $P_N$ to refer to the marginal of $P_H$ and $P_N$ on the span of $\bm w$.
    Since the densities agree in the space orthogonal to $\bm w$, and since they factorize over these two spaces, we obtain
    \begin{align*}
        \TVD{P_H}{P_N} = \;\sup_{\mathclap{A \sse \Span(\bm w)}} \;\Abs{P_H(A) - P_N(A)} \,.
    \end{align*}
    For ease of notation, we identify the span of $w$ with the real line.
    Further, for $k \in \Z$, let $P_{H,k}$ and $P_{N,k}$ denote the density of the $k$-th component of $P_H$ and $P_N$ respectively.
    Further, let $w_k$ denote the weight of the $k$-th component - which is the same in both cases.
    It follows that
    \begin{align*}
        \TVD{P_H}{P_N} &= \sup_{A \sse \R} \,\Abs{P_H(A) - P_N(A)} = \sup_{A \sse \R} \,\Abs{\sum_{k \in \Z} w_k \Paren{P_{H,k}(A) - P_{N,k}(A)}} \\
        &\leq \sum_{k \in \Z} w_k \cdot \sup_{A \sse \R} \,\Abs{P_{H,k}(A) - P_{N,k}(A)}
    \end{align*}
    Next, fix $k \in \Z$ and let $I_k$ denote the support of $P_{N,k}$.
    Let $Z = P_{H,k}(I_k)$ and recall that for $C \subseteq I_k$ it holds that $P_{N,k}(C) = \tfrac 1 Z \cdot P_{H,k}(C)$ and for $C$ disjoint from $I_k$ that $P_{N,k}(C) = 0$.
    We can then bound
    \begin{align*}
        \sup_{A \sse \R} \,\Abs{P_{H,k}(A) - P_{N,k}(A)} &= \sup_{A \sse \R}\, \Abs{P_{H,k}(A \cap I_k) + P_{H,k}(A \cap I_k^c) - P_{N,k}(A \cap I_k) - P_{N,k}(A \cap I_k^c)} \\
        &\leq \sup_{A \sse \R} \,\Abs{P_{H,k}(A \cap I_k)  - P_{N,k}(A \cap I_k)} + \sup_{A \sse \R} \,\Abs{P_{H,k}(A \cap I_k^c) - P_{H,k}(A \cap I_k^c)} \\
        &= \frac{1-Z}{Z} \cdot P_{H,k}(I_k) + P_{H,k}(I_k^c) = 2 \cdot \Paren{1-Z} \,.
    \end{align*}

    Let $\mu_k = \tfrac{\gamma}{\beta^2 + \gamma^2} (k - c)$ and $\sigma_k^2 = \tfrac 1 {2\pi} \cdot \tfrac{\beta^2}{\beta^2 + \gamma^2}$ and denote by $X_k$ the random variable distributed according to $P_{H,k}$.
    Note that $X_k \sim N(\mu_k, \sigma_k^2)$ and $I_k = [\mu_k - \alpha, \mu_k - \alpha]$, where $\alpha = \tfrac 1 {10} \cdot \tfrac {\gamma}{\gamma^2 + \beta^2}$.
    It follows that
    \begin{align*}
        1-Z = \Psymb\Paren{\Abs{X_k} \geq \alpha} &\leq 2 \cdot \exp \Paren{- \frac{\alpha^2}{2\sigma_k^2}} = 2 \cdot \exp\Paren{-\frac{2 \pi \cdot \gamma^2 \cdot \Paren{\beta^2 + \gamma^2}}{200 \beta^2 \cdot \Paren{\beta^2 + \gamma^2}^2}} \leq 2 \cdot \exp\Paren{- \frac{ \gamma^2}{50\beta^2\cdot \Paren{\beta^2 + \gamma^2}}} \\
        & \leq2 \cdot \exp\Paren{- \frac{1}{100\beta^2}}\,.
    \end{align*}

    Hence, we finally obtain that $$\TVD{P_H}{P_N} \leq \Paren{\sum_{k \in \Z} w_k} \cdot 4 \cdot \exp\Paren{- \frac{1}{100\beta^2}} = 4 \cdot \exp\Paren{- \frac{1}{100\beta^2}}\,.$$
\end{proof}

\begin{lemma}
    \label{lem:tvd_distribution_specific}
    Let $D,D'$ be distributions over $\R^M \times \Set{-1,+1}$ with densities defined below, then $\TVD{D}{D'} \leq \tfrac{1}{\poly(M)}$.
    The densities are equal to
    \begin{align*}
        p_D(\bm x, y) &=  \frac 1 \beta \rho(\bm x) \cdot \begin{cases}
            \sum_{k \in \Z} \int_0^{1/2} \rho_{\beta}\Paren{ c +k - \gamma \iprod{\bm w, \bm x}} \, dc\,, &\quad \text{if $y = +1$,} \\
            \sum_{k \in \Z} \int_{1/2}^1 \rho_{\beta}\Paren{ c +k - \gamma \iprod{\bm w, \bm x}} \, dc\,, &\quad \text{if $y = -1$.} 
        \end{cases}
    \end{align*}
    and
    \begin{align*}
        p_{D'}(\bm x, y) &= \rho(\bm x) \cdot \begin{cases} \sum_{k \in \Z} \Ind\Paren{\gamma \iprod{\bm w, \bm x} \in [k, k+1/2)}\,, &\quad \text{if $y = +1$,} \\
            \sum_{k \in \Z} \Ind\Paren{\gamma \iprod{\bm w, \bm x} \in [k+1/2, k+1)} \,, &\quad \text{if $y = -1$.} 
        \end{cases}
    \end{align*}
\end{lemma}

\begin{proof}
    The proof proceeds similary to \cref{lem:tvd_hclwe_and_nhclwe}.
    First, note that by symmetry
    \begin{align*}
        \TVD{D}{D'} &= \max\Set{\sup_{A \sse \R^M} \Abs{\Psymb_D(\bm x \in A, y = \ell) - \Psymb_{D'}(\bm x \in A, y = \ell)} \suchthat \ell \in \Set{-1,+1}} \\
        &= \sup_{A \sse \R^M} \Abs{\Psymb_D(\bm x \in A, y = +1) - \Psymb_{D'}(\bm x \in A, y = +1)} \\
        &= \sup_{A \sse \Span(\bm w)} \Abs{\Psymb_D(\iprod{\bm w, \bm x} \in A, y = +1) - \Psymb_{D'}(\iprod{\bm w, \bm x} \in A, y = +1)}\,.
    \end{align*}
    Without loss of generality, identfiy the span of $\bm w$ with the real line.
    Abusing notation, we denote the one-dimensional densities by $p_D$ and $p_{D'}$ as well.
    Define $I_k(z) \coloneqq \int_0^{1/2} \rho_{\beta}\Paren{ c +k - \gamma z} \, dc$ and observe that
    \begin{align*}
        p_D(z) = \frac{1}{\beta} \rho(z) \cdot \sum_{k \in \Z} I_k(z) \,, \quad&&\quad
        p_{D'}(z) = \rho(z) \cdot \sum_{k \in \Z} \Ind\Paren{\gamma z \in [k,k+1/2]} \,.
    \end{align*}
    It follows that $\TVD{D}{D'} = \TVD{p_D}{p_{D'}}$ (referring to the one-dimensional densities).
    To bound this quantity, we introduce the following intermediate distribution $\tilde{D}$ defined as follows:
    For $z \in \R$ let $k^*(z) = \argmin_{k \in \Z} \min\Set{\abs{\gamma z - k}, \abs{\gamma z - k - 1/2}}$, then we set 
    \[
        p_{\tilde{D}}(z) \propto \frac{1}{\beta} \rho(z) \cdot I_{k^*(z)}(z) \,.
    \]
    We will first show that $\TVD{p_D}{p_{\tilde{D}}} \leq \exp\Paren{-\poly(M)}$ and second that $\TVD{p_{\tilde{D}}}{p_{D'}} \leq \tfrac{1}{\poly(M)}$ which together imply the desired result.
    We will use that for measure $P,Q$ it holds that $\TVD{P}{Q} \leq \sqrt{2} \HD{P}{Q}$, where $\HD{P}{Q}$ is the Hellinger distance defined as $\HD{P}{Q} = \tfrac 1 2 \int \Paren{\sqrt{p(x)} - \sqrt{q(x)}}^2 \, dx$ for $p,q$ the densitites of the measures $P$ and $Q$ respectively.
    
    Let $Z$ be the normalization constant in the density $p_{\tilde{D}}$, it follows that 
    \[
        Z = \int_{\R} \frac{1}{\beta} \rho(z) \cdot I_{k^*(z)}(z) \, dz = 1 - \int_{\R} \frac{1}{\beta} \rho(z) \cdot \sum_{k \neq k^*(z)} I_k(z) \, dz \,.
    \]
    Note that for a given $z$ and $k \neq k^*(z)$ it holds that $\gamma z$ is at distance at least $\tfrac{\abs{k^*(z) - k}}{4}$ from the interval $[k, k+1/2]$.
    Hence, we can bound $I_k(z)$ as follows:
    \[
        I_k(z) = \int_0^{1/2} \rho_{\beta}\Paren{ c +k - \gamma z} \, dc \leq \frac 1 2 \rho_\beta \Paren{\tfrac {\abs{k^*(z) - k}} 4} = \exp\Paren{-\Paren{k^*(z) - k}^2\poly(n)} \,.
    \]
    It follows that
    \[
        \sum_{k \neq k^*(z)} I_k(z) \leq \sum_{k \geq 1} \exp\Paren{- k \cdot \poly(n)} \leq \exp\Paren{-\poly(n)} \,,  
    \]
    implying that $Z \geq 1 - \exp\Paren{-\poly(n)}$.
    Using this we can bound the Hellinger distance
    \begin{align*}
        \HD{D}{\tilde{D}} &= \frac 1 {2\beta} \int_{\R} \rho(z) \Paren{\sqrt{\frac{I_{k^*(z)}(z)}{Z}} - \sqrt{\sum_{k \in \Z} I_k(z)}}^2 \, dz \\
        &\leq \frac{1}{\beta} \Brac{\int_{\R} \rho(z) \Paren{\sqrt{\frac{I_{k^*(z)}(z)}{Z}} - \sqrt{I_{k^*(z)}(z)}}^2 \, dz + \int_{\R} \rho(z) \Paren{\sqrt{I_{k^*(z)}(z)} - \sqrt{\sum_{k \in z} I_k(z)}}^2 \, dz } \\
        &\leq \exp\Paren{-\poly(n)} \,,
    \end{align*}
    where in the last inequality we used that $\int_{\R} \rho(z) \, dz = 1$ and $I_{k^*(z)}(z) \leq 1$ for all $z$.

    Next, we turn to bound $\HD{\tilde{D}}{D'}$.
    To this end, let $\tau > 0$ be some parameter to be chosen later and note that $\sum_{k \in \Z} \Ind\Paren{\gamma z \in [k,k+1/2]} = \Ind\Paren{\gamma z \in [k^*(z),k^*(z)+1/2]}$.
    We obtain
    \begin{align*}
        \HD{\tilde{D}}{D'} &= \frac{1}{2} \int_{\R} \rho(z)\Paren{\sqrt{\frac 1 \beta I_{k^*(z)}(z)} - \Ind\Paren{\gamma z \in [k^*(z),k^*(z)+1/2]}}^2 \, dc
    \end{align*}
    To begin with, note that $$\frac 1 \beta I_{k^*(z)}(z) \leq \int_{\R} \frac 1 \beta \rho_\beta\Paren{c+k-\gamma z} \, dc \leq 1 \,.$$
    We proceed by making a case distinction.
    First, consider $z$ such that $\gamma z \in [k^*(z)+\tau,k^*(z)+1/2-\tau]$ and let $X \sim N(0, \tfrac \beta {2\pi})$.
    For such $z$ it holds that
    \begin{align*}
        \frac 1 \beta I_{k^*(z)}(z) &= \int_0^{1/2} \frac 1 \beta \rho_{\beta}\Paren{ c + k^*(z) - \gamma z} \, dc = \Psymb\Paren{k^*(z) - \gamma z \leq X \leq k^*(z) + \tfrac 1 2 - \gamma z} \\
        &\geq 1 - \Psymb\Paren{\Abs{X} \geq \tau} \geq 1 - 2 \exp\Paren{-\tfrac{\pi \tau^2}{\beta^2}} \,.
    \end{align*}
    Next, consider $z$ such that $\min\Set{\abs{\gamma z - k^*(z)}, \abs{\gamma z - k^*(z) - \tfrac 1  2}} \geq \tau$.
    We obtain that
    \[
        \frac 1 \beta I_{k^*(z)}(z) \leq \Psymb\Paren{\Abs{X} \geq \tau} \leq 2 \exp\Paren{-\tfrac{\pi \tau^2}{\beta^2}} \,.
    \]
    Let $S = \bigcup_{k \in \Z} [k-\tau, k+\tau] \cup [k+ \tfrac 1 2 - \tau, k + \tfrac 1 2 + \tau]$.
    Using the above, we can bound
    \[
        \frac 1 2 \int_{\R \setminus S} \rho(z)\Paren{\sqrt{\frac 1 \beta I_{k^*(z)}(z)} - \Ind\Paren{\gamma z \in [k^*(z),k^*(z)+1/2]}}^2 \, dc \leq \exp\Paren{-\frac{\pi \tau^2}{\beta^2}} \,.
    \]
    It remains to bound the integral on $S$.
    For this, note that $\rho$ is symmetric around $z = 0$ and monotone for $z \geq 0$ and $z \leq 0$.
    This yields
    \begin{align*}
        &\frac{1}{2} \int_S \rho(z)\Paren{\sqrt{\frac 1 \beta I_{k^*(z)}(z)} - \Ind\Paren{\gamma z \in [k^*(z),k^*(z)+1/2]}}^2 \, dc \\
        &\leq 4 \Brac{\int_{0}^{\tau} \rho(z) \, dz + \int_{1/2 -\tau}^{1/2 + \tau} \rho(z) \, dz  + \sum_{k \geq 1} \int_{k-\tau}^{k+\tau} \rho(z) \, dz + \int_{k + 1/2 -\tau}^{k+ 1/2 + \tau} \rho(z) \, dz }\\
        &\leq 4 \Brac{\int_{0}^\tau \rho(z) \, dz + \frac{1}{\tfrac 1 {4\tau} -1} \int_\tau^{1/2 + \tau} \rho(z) \, dz + \frac{1}{\tfrac 1 {4\tau} -1} \sum_{k \geq 1} \int_{k-1/2 + \tau}^{k + \tau} \rho(z) \, dz + \int_{k+\tau}^{k+1/2+\tau} \rho(z) \, dz } \\
        &\leq 12 \tau \int_0^\infty \rho(z) \, dz = 6 \tau \,.
    \end{align*}
    Hence, combining the above bounds and choosing $\tau = \sqrt{\beta} = \tfrac 1 {\poly(M)}$, we obtain
    \[
        \HD{\tilde{D}}{D'} \leq \exp\Paren{-\frac{\pi \tau^2}{\beta^2}} + 6\tau = \exp\Paren{-\poly(M)} + \frac{1}{\poly(M)}
    \]
    as desired.

\end{proof}

\subsection{Supporting Lemmas about Optimal Halfspaces}

\begin{lemma}
    \label{lem:distribution_specific_optimal_ltf}
    Consider the distribution $D'$ over $\R^M \times \Set{-1,1}$ with density given by
    \begin{align*}
        p_{D'}(\bm x, y) &= \rho(\bm x) \cdot \begin{cases} \sum_{k \in \Z} \Ind\Paren{\gamma \iprod{\bm w, \bm x} \in [k, k+1/2)}\,, &\quad \text{if $y = +1$,} \\
            \sum_{k \in \Z} \Ind\Paren{\gamma \iprod{\bm w, \bm x} \in [k+1/2, k+1)} \,, &\quad \text{if $y = -1$.} 
        \end{cases}
    \end{align*}
    Let $f^*(\bm x) = \sign\Paren{\iprod{\bm w, \bm x}}$, then there exists an absolute constant $c> 0$ such that $\err_{D'}\Paren{f^*} \leq \tfrac 1 2 - \tfrac c \gamma$.

    Further, for each $\ell \geq 2$ there exists a degree-$\ell$ PTF $h^*$ such that $\err_{D'} \Paren{h^*} \leq \tfrac 1 2 - \tfrac {c' \ell} \gamma$, for some absolute constant $c' > 0$.
\end{lemma}
\begin{proof}
    First note, that for $(\bm x, y) \sim D'$ only depends on $\iprod{\bm w, \bm x}$.
    Let $z = \iprod{\bm w, \bm x}$ and $A_k = [\tfrac k \gamma, \tfrac{k+1/2} \gamma], B_k = [\tfrac{k+1/2} \gamma, \tfrac {k+1} \gamma]$ for $k \in \Z$.
    Further, let $X \sim N(0, \tfrac 1 {2 \pi})$.

    We first prove the result about linear threshold functions.
    By symmetry it holds that
    \begin{align*}
        \err_{D'}\Paren{f^*} &= 2 \Psymb\Paren{f^*(\bm x) \neq y, y = -1} = 2 \sum_{k \in \Z} \int_{B_k} \Ind\Paren{f^*(z) \neq -1} \rho(z) \, dz = 2 \sum_{k \geq 0} \Psymb\Paren{X \in B_k}
    \end{align*}
    Note that by construction, for $k \geq 0$, it holds that $2\Psymb\Paren{X \in B_k} \leq \Psymb\Paren{X \in B_k} + \Psymb\Paren{X \in A_k}$.
    Let $k^*$ be a non-negative integer to be chosen later and assume that there exists $\e = \e(k^*)$ such that $2\Psymb\Paren{X \in B_k} \leq (1-\e) \Brac{\Psymb\Paren{X \in B_k} + \Psymb\Paren{X \in A_k}}$, then for $k^* \geq \gamma$
    \begin{align*}
        2\sum_{k \geq 0} \Psymb\Paren{X \in B_k} &\leq \sum_{0 \leq k < k^*} \Psymb\Paren{X \in B_k} + \Psymb\Paren{X \in A_k} + \Paren{1-\e} \cdot \sum_{k \geq k^*} \Psymb\Paren{X \in B_k} + \Psymb\Paren{X \in A_k} \\
        &= \frac 1 2 - \Psymb\Paren{X \geq \frac {k^*} \gamma} + \Paren{1-\e} \cdot \Psymb\Paren{X \geq \frac {k^*} \gamma} = \frac 1 2 - \e \cdot \Psymb\Paren{X \geq \frac {k^*} \gamma} \\
        &\leq \frac 1 2 - \e \cdot \frac{k^*/\gamma}{2 \pi \Paren{k^*/\gamma}^2 + 1}\exp\Paren{-\frac{\pi \cdot (k^*)^2}{\gamma^2}} \\
        &\leq \frac 1 2 -  \e \cdot \frac \gamma {4 \pi k^*} \exp\Paren{-\frac{\pi \cdot (k^*)^2}{\gamma^2}} \,,
    \end{align*}
    where we also used standard bounds for the pdf of the standard Gaussian distribution.
    Next, we aim to find $\e$ and calculate
    \begin{align*}
        \Psymb\Paren{X \in A_k} - \Psymb\Paren{X \in B_k} &= \int_{k/\gamma}^{(k+1/2)/\gamma} \exp\Paren{-\pi z^2} \, dz - \int_{(k+1/2)/\gamma}^{(k+1)/\gamma} \exp\Paren{-\pi z^2} \, dz \\
        &= \int_{(k+1/2)/\gamma}^{(k+1)/\gamma} \Brac{\exp\Paren{-\pi (z -1/(2\gamma))^2} - \exp\Paren{-\pi z^2}} \, dz \\
        &= \int_{(k+1/2)/\gamma}^{(k+1)/\gamma} \exp\Paren{-\pi z^2} \cdot \Brac{\exp\Paren{\pi \frac z \gamma} \exp\Paren{- \frac{\pi}{4\gamma^2}} - 1 } \, dz
    \end{align*}
    For $k \geq 1$ we can bound
    \begin{align*}
        \exp\Paren{\pi \frac z \gamma} \exp\Paren{- \frac{\pi}{4\gamma^2}} &\geq \Paren{1 + \pi \frac z \gamma} \Paren{1 - \frac{\pi}{4 \gamma^2}} = 1 + \pi \frac{z}{\gamma} - \frac{\pi}{4 \gamma^2} - \pi^2 \frac z {4\gamma^3} \\
        &\geq 1 + \pi \frac z {2\gamma} - \frac{\pi}{4 \gamma^2} \geq 1 + \pi \frac k {4 \gamma^2} \,.
    \end{align*}
    Which implies
    \[
        \Psymb\Paren{X \in A_k} - \Psymb\Paren{X \in B_k} \geq \pi \frac k {4 \gamma^2} \Psymb\Paren{X \in B_k} \,.
    \]
    Rearringing implies that
    \[
        \Psymb\Paren{X \in B_k} \leq \frac{2}{2+\pi \tfrac k {4 \gamma^2}} \cdot \Paren{\Psymb\Paren{X \in B_k} + \Psymb\Paren{X \in A_k}}
    \]
    which yields $\e(k) \geq \pi \frac k {12 \gamma^2} \geq \frac k {4 \gamma^2}$.
    Hence, for $k^* = \lceil \gamma \rceil$ we have
    \[
        \e(k^*) \cdot \frac \gamma {4 \pi k^*} \exp\Paren{-\frac{\pi \cdot (k^*)^2}{\gamma^2}} \geq \frac 1 {4\gamma} \cdot \frac 1 {4 \pi} \cdot \exp\Paren{-4\pi} \,.
    \]
    As desired, this implies that for an absolute constant $c > 0$ it holds that
    \[
        \err_{D'}\Paren{f^*} \leq \frac 1 2 - \frac{c}{\gamma} \,.
    \]

    Next, we prove the result about degree-$\ell$ PTFs.
    Again, since the labels of $D'$ only depend on the direction $\bm w$ it suffices to define a one-dimensional degree-$\ell$ polynomial $p_z \from \R \to \R$.
    The final PTF will be defined as $h^*(\bm x) = \sign{p_z(\iprod{\bm w, \bm x})}$.
    Note that $p_z$ can be fully specified by $\ell$ roots and the sign it takes between any two roots.
    For simplicity, we assume that $\ell$ is odd and consider degree-$2\ell +1$ PTFs, the even case works analogously.
    Let $p_z$ be the polynomial that has roots $-\tfrac \ell {2 \gamma}, - \tfrac {\ell-1} {2\gamma}, \ldots, 0, \ldots, \tfrac {\ell - 1} {2\gamma}, \tfrac \ell {2\gamma}$.
    Further, let its sign be positive between 0 and $\tfrac 1 {2\gamma}$ and alternate on the other intervals.
    Again, for simplicity and without loss of generality, also assume that its sign after the greatest positive root is positive.
    Observe that this implies that $\ell$ is even.
    Note that for $\ell = 0$ we recover the LTF from above.
    Let $c > 0$ be some absolute constant.
    By symmetry and the results above it follows that, note that we use that $h^*$ agrees with the label of all samples $(\bm x, y)$ such that $\iprod{\bm w, \bm x} \leq 0$ and $y = -1$. 
    \begin{align*}
        \err_{D'}\Paren{f^*} &= 2 \Psymb\Paren{f^*(\bm x) \neq y, y = -1} = 2 \sum_{k \in \Z} \int_{B_k} \Ind\Paren{f^*(z) \neq -1} \rho(z) \, dz = 2 \sum_{k \geq \ell/2} \Psymb\Paren{X \in B_k} \\
        &= 2 \sum_{k \geq 0} \Psymb\Paren{X \in B_k} - 2 \sum_{k < \ell/2} \Psymb\Paren{X \in B_k} \leq \frac 1 2 - \frac c \gamma - \Psymb\Paren{\frac 1 \gamma \leq X \leq \frac{\ell/2 + 1}{\gamma}} \\&
        \leq \frac 1 2 - \frac c \gamma - \frac 1 2 \cdot \Psymb\Paren{X \leq \frac {\ell + 1} {2\gamma}} \,.
    \end{align*}
    Using that $\ell / \gamma \leq 1$ we bound as before
    \begin{align*}
        \tfrac 1 2 \cdot \Psymb\Paren{X \leq \frac {\ell + 1} {2\gamma}} &\geq \frac{(\ell + 1)/\gamma}{ 4 \pi \Paren{(\ell +1 )/ \gamma }^2 + 1} \exp\Paren{- \frac{\pi \cdot (\ell+1)^2}{\gamma^2}} \\
        &\geq \frac{\ell + 1}{ 36 \gamma} \exp\Paren{- \pi} \,.
    \end{align*}
    Hence, there exists an absolute constant $c' > 0$ such that
    \[
        \err_{D'}\Paren{h^*} \leq \frac 1 2 - \frac {c' \ell} \gamma \,. 
    \]
\end{proof}

Next, we prove \cref{lem:hard_ptf_implies_hard_ltf}.
\restatelemma{lem:hard_ptf_implies_hard_ltf}

\begin{proof}
    We start by describing the mapping $\phi$.
    Denote by $\alpha = \Paren{\alpha_1, \ldots, \alpha_n} \in \N^n$ a multi-index and by $\abs{\alpha} = \sum_{i=1}^n \alpha_i$ its size.
    Let $M' = \binom{n+d}{n}$ and let
    \begin{align*}
        \phi \colon \R^n &\rightarrow \R^M\,, \\
        \bm x &\mapsto \Paren{\Paren{{\bm x}^\alpha}_{\abs{\alpha}\leq d}, \bm 0} \,,
    \end{align*}
    where by $\bm 0$ we mean the vector containing $M - M'$ zeros.
    Define the distribution $D'$ over $\R^M \times \Set{-1,+1}$ as first drawing $(\bm x,y)\sim D$ and then outputting $\Paren{\phi(\bm x),y}$.
    Clearly, restricted to the support of $D$, the map $\phi$ is a bijection between $\supp\Paren{D}$ and $\supp\Paren{D'}$.

    Next, consider any degree-$d$ polynomial threshold function $h \colon \R^n \rightarrow \Set{-1,+1}$.
    Since $M' = \binom{n+d}{n} \leq M$ there exists a linear threshold function $f$ such that for all $\bm x \in \R^n$ it holds that $h(\bm x) = f(\phi(\bm x))$.
    It follows that $$\err_{D'}\Paren{f'} = \Psymb_{(\bm x',y') \sim D'} \Paren{f(\bm x') \neq y'} = \Psymb_{(\bm x,y) \sim D} \Paren{f(\phi(\bm x)) \neq y} = \Psymb_{(\bm x,y) \sim D} \Paren{h(\bm x) \neq y} = \err_D\Paren{f}\,.$$
    Similarly, for every binary function $f \colon \supp\Paren{D'} \rightarrow \Set{-1,+1}$ we can define the binary function $h \colon \R^n \rightarrow \Set{-1,+1}$ such that $h \Paren{\bm x} = f\Paren{\phi\Paren{\bm x}}$.
    Hence, we have $$\err_{D'}\Paren{f'} = \err_D\Paren{f}\,.$$
    Since in both cases we have to consider at most $M$ coefficients we can compute the linear/polynomial threshold function in time $\poly\Paren{M}$.
    Moreover, in both cases, for $\bm {\tilde{x}'} \in \supp\Paren{D'}$ it holds that 
    \begin{align*}
        \Psymb_{(\bm x',y') \sim D'} \Paren{f(\bm x') \neq y' \suchthat \bm x' = \bm {\tilde{x}'}} &= \Psymb_{(\bm x,y) \sim D} \Paren{f(\phi\Paren{\bm x}) \neq y \suchthat \phi\Paren{\bm x} = \bm {\tilde{x}'}} \\
        &= \Psymb_{(\bm x,y) \sim D} \Paren{h(\bm x) \neq y \suchthat \bm x = \phi^{-1}\Paren{\bm {\tilde{x}'}}} \,,
    \end{align*}
    as desired.

\end{proof}

\subsection{Small Facts}

\begin{lemma}
    \label{lem:small_tvd_advantage}
    Let $n \in \N, \e > 0$ and distributions $D_n^0$ and $D_n^1$ be such that there exists no $T$-time distinguisher with advatage at least $\e$ between $D_n^0$ and $D_n^1$.
    Further, let $D_n^{1'}$ be a third distribution such that $\TVD{D_n^1}{D_n^{1'}} = \negl(n)$.
    Then there exists no $T$-time distingiusher with advantage at least $\e - \negl(n)$ between $D_n^0$ and $D_n^{1'}$.
\end{lemma}
\begin{proof}
    Suppose there exists a distinguisher $\cA$ between $D_n^0$ and $D_n^{1'}$ with advantage at least $\e - \negl(n)$.
    Using this distinguisher to distinguish between $D_n^0$ and $D_n^1$ gives advantage
    \begin{align*}
        \Abs{\Psymb_{x \sim D_n^0} \Paren{\cA(x) = 0} - \Psymb_{x \sim D_n^1} \Paren{\cA(x) = 0}} \geq \Abs{\Psymb_{x \sim D_n^0} \Paren{\cA(x) = 0} - \Psymb_{x \sim D_n^{1'}} \Paren{\cA(x) = 0}} + \negl(n) \geq \e
    \end{align*}
    which is a contradiction.
\end{proof}

\begin{fact}[Poisson Summation Formula]
    \label{fact:poisson_summation}
    For any lattice $L$ and any function $f$ it holds that $$f(L) = \det\Paren{L^*} \cdot \hat{f}\Paren{L^*}$$ where $L^* = \Set{\bm y \in \R^n \suchthat \iprod{\bm x, \bm y} \in \Z \text{ for all } \bm x \in \Z}$ is the dual lattice of $L$ and $\hat{f}$ the Fourier transform of $f$.
\end{fact}

\begin{fact}[\cite{peikert_rho_fact}]
    \label{fact:rho_fact}
    For any $r_1, r_2 > 0$ and vectors $\bm x, \bm c_1, \bm c_2 \in \R^n$, let $r_0 = \sqrt{r_1^2 + r_2^2}, r_3 = \tfrac {r_1 r_2}{r_0}$, and $\bm c_3 = \tfrac {r_3^2}{r_1^2} \bm c_1 + \tfrac {r_3^2}{r_2^2} \bm c_2$. Then $$\rho_{r_1}\Paren{\bm x - \bm c_1} \cdot \rho_{r_2}\Paren{\bm x - \bm c_2} = \rho_{r_0} \Paren{c_1 - c_2} \cdot \rho_{r_3}\Paren{\bm x - \bm c_3} \,.$$
\end{fact}

\begin{fact}
    \label{fact:eq_hclwe}
    Let $\gamma,\beta \geq 0$ and $\bm w \in \R^d$, then $$\sum_{k \in \Z} \rho_{\sqrt{\beta^2 + \gamma^2}}(k \,; c) \cdot \rho\Paren{\pi_{{\bm w}^\perp}(\bm y)} \cdot \rho_{\beta / \sqrt{\beta^2 + \gamma^2}}\Paren{\iprod{\bm w, \bm y} \,; \frac{\gamma}{\beta^2 + \gamma^2} (k - c)} = \rho(\bm y) \cdot \sum_{k \in \Z} \rho_{\beta} \Paren{\gamma \iprod{\bm w, \bm y} \,;  k - c} \,.$$
\end{fact}

\begin{proof}
    Clearly, for $\bm y$ orthogonal to $\bm w$ the equality holds.
    Consider any $\bm y$ in the span of $\bm w$ and for convenience write $z = \iprod{\bm y,\bm w}$.
    Fix $k \in \Z$ then we have that
    \begin{align*}
        \rho_{\sqrt{\beta^2 + \gamma^2}}(k \,; c) \cdot \rho_{\beta / \sqrt{\beta^2 + \gamma^2}}\Paren{z \,; \frac{\gamma}{\beta^2 + \gamma^2} (k-c) } = \exp\Paren{-\pi \Brac{\frac{\Paren{k-c}^2}{\beta^2 + \gamma^2} + \frac{\Paren{\beta^2 + \gamma^2} \cdot\Paren{z - \tfrac \gamma {\beta^2 + \gamma^2} (k-c)}^2}{\beta^2}}} \,.
    \end{align*}
    Focusing only on the expression inside the exponential function (and ignoring the $\pi$) we obtain
    \begin{align*}
        \frac{\Paren{k-c}^2}{\beta^2 + \gamma^2} + \frac{\Paren{\beta^2 + \gamma^2} \cdot\Paren{z - \tfrac \gamma {\beta^2 + \gamma^2} (k-c)}^2}{\beta^2} &= \frac{\Paren{k-c}^2 \cdot \beta^2 + \Brac{\Paren{\beta^2 + \gamma^2} \cdot z - \gamma \cdot \Paren{k-c}}^2}{\Paren{\beta^2 + \gamma^2}\cdot \beta^2}\\
        &= \frac{\Paren{\beta^2 + \gamma^2} \cdot z^2 + \Paren{k-c}^2 - 2 \cdot \Paren{k-c} \cdot \gamma\cdot z }{\beta^2} \\
        &= \frac{\Paren{\Paren{k-c} - \gamma \cdot z}^2}{\beta^2} + z^2 \,.
    \end{align*}
    Hence, it follows that
    \begin{align*}
        \rho_{\sqrt{\beta^2 + \gamma^2}}(k \,; c) \cdot \rho_{\beta / \sqrt{\beta^2 + \gamma^2}}\Paren{z \,; \frac{\gamma}{\beta^2 + \gamma^2} (k-c) } = \rho_\beta \Paren{\gamma \cdot z \,; k-c} \cdot \rho(z)
    \end{align*}
    which implies the claim.
\end{proof}

\end{document}